\documentclass[11pt]{article}
\usepackage[paper=letterpaper,margin=1in]{geometry}

\usepackage[numbers]{natbib}
\usepackage{times}
\usepackage{etoc}
\usepackage{paralist}
\usepackage{amsmath,amssymb,amsthm,dsfont}
\usepackage{forloop}

\usepackage[utf8]{inputenc} %
\usepackage[T1]{fontenc}    %
\usepackage{hyperref}       %
\usepackage{url}            %
\usepackage{booktabs}       %
\usepackage{amsfonts}       %
\usepackage{nicefrac}       %
\usepackage{microtype}      %
\usepackage{thmtools}
\usepackage{thm-restate}
\usepackage{graphicx}

\usepackage{xcolor}
\usepackage[linesnumbered,ruled]{algorithm2e}
\usepackage[capitalize]{cleveref}

\newtheorem{theorem}{Theorem}
\newtheorem{lemma}[theorem]{Lemma}
\newtheorem{cor}[theorem]{Corollary}
\newtheorem{prop}[theorem]{Proposition}

\theoremstyle{definition}
\newtheorem{defn}{Definition}
\newtheorem{remark}{Remark}
 
\newcommand{\defcal}[1]{\expandafter\newcommand\csname 
	c#1\endcsname{{\mathcal{#1}}}}
\newcommand{\defbb}[1]{\expandafter\newcommand\csname 
	b#1\endcsname{{\mathbb{#1}}}}
\newcommand{\defbf}[1]{\expandafter\newcommand\csname 
	bf#1\endcsname{{\mathbf{#1}}}}

\newcounter{calBbCounter}

\forLoop{1}{26}{calBbCounter}{
	\edef\letter{\Alph{calBbCounter}}
	\expandafter\defcal\letter
	\expandafter\defbb\letter
	\expandafter\defbf\letter
}

\forLoop{1}{26}{calBbCounter}{
	\edef\letter{\alph{calBbCounter}}
	\expandafter\defbf\letter
}

\newcommand{\ind}{\mathds{1}}

\newcommand{\roco}{\cR_{\textnormal{OCO}}}
\newcommand{\rtk}{\mathcal{R}(T, K)}
\DeclareMathOperator{\sign}{sign}

\newcommand{\ie}{{\it i.e.}}

\SetKwInOut{Input}{Input}
\SetKwInOut{Output}{Output}

\crefname{prop}{Proposition}{Propositions}

\title{Minimax Regret of Switching-Constrained Online Convex Optimization: No Phase Transition}

\author{%
  Lin~Chen$^{1,2}$\thanks{First two authors contributed equally. Correspondence to: Lin~Chen <\href{mailto:lin.chen@berkeley.com}{lin.chen@berkeley.edu}>.} \quad Qian~Yu$^{3*}$ \quad  Hannah~Lawrence$^4$ \quad Amin~Karbasi$^1$ \\
  $^1$ Yale University \quad $^2$ Simons Institute for the Theory of Computing\\
  $^3$ University of Southern California\quad 
  $^4$ Massachusetts Institute of Technology\\
}
\date{}

\begin{document}

\maketitle

\begin{abstract}
  We study the problem of switching-constrained online convex optimization (OCO), where the player has a limited number of opportunities to change her action. While the discrete analog of this online learning task has been studied extensively, previous work in the continuous setting has neither established the minimax rate nor algorithmically achieved it. In this paper, we show that $ T $-round switching-constrained OCO with fewer than $ K $ switches has a minimax regret of $ \Theta(\frac{T}{\sqrt{K}}) $. In particular, it is at least $ \frac{T}{\sqrt{2K}} $ for one dimension and at least $ \frac{T}{\sqrt{K}} $ for higher dimensions. The lower bound in higher dimensions is attained by an orthogonal subspace argument. In one dimension, a novel adversarial strategy yields the lower bound of $O(\frac{T}{\sqrt{K}})$, but a precise minimax analysis including constants is more involved. To establish the tighter one-dimensional result, we introduce the \emph{fugal game} relaxation, whose minimax regret lower bounds that of switching-constrained OCO. We show that the minimax regret of the fugal game is at least $ \frac{T}{\sqrt{2K}} $ and thereby establish the optimal minimax lower bound in one dimension. To establish the dimension-independent upper bound, we next show that a mini-batching algorithm provides an $ O(\frac{T}{\sqrt{K}}) $ upper bound, and therefore conclude that the minimax regret of switching-constrained OCO is $ \Theta(\frac{T}{\sqrt{K}}) $ for any $K$. This is in sharp contrast to its discrete counterpart, the switching-constrained prediction-from-experts problem, which exhibits a phase transition in minimax regret between the low-switching and high-switching regimes. 
\end{abstract}

\section{Introduction}

\label{sec:introduction}
Online learning provides a versatile framework for studying a wide range of 
dynamic optimization problems, with manifold applications in 
portfolio selection~\citep{li2014online}, packet routing~\citep{awerbuch2008online}, hyperparameter optimization~\citep{LiHyperparameter}, and spam filtering~\citep{sculley2007relaxed}. The fundamental problem is typically formulated as a repeated game 
between a player and an adversary. In the $t^{th}$ round, the player first chooses an 
action $x_t$ from the set 
of all possible actions $\cD$; the adversary then responds by revealing the penalty for that action, a function $f_t \: : \: \cD \rightarrow \mathbb{R}$. The player's goal is to minimize the total penalties she receives, while the adversary's goal is to maximize the penalties she assigns to the player's 
action. Explicitly, the standard benchmark for success is regret, the difference between the player's accumulated penalty and that of the best fixed action in hindsight:
$ \mathcal{R}=\sum_{i=1}^Tf_i(x_i) - \inf_{x \in \cD} \sum_{i=1}^Tf_i(x) \,.$

Several variants of this general learning setting have been studied. When $\cD$  is a discrete set of actions, the game is called either ``prediction from experts" (PFE), if the player is allowed knowledge of the complete function $f_t(\cdot)$ on each round, or ``multi-armed bandit" (MAB), if only $f_t(x_t)$ is revealed on each round. Crucially, the adversary is not strongly adaptive and picks $f_t$ based solely on prior knowledge of the player's randomized strategy and $x_1,\dots,x_{t-1}$, and not $x_t$; otherwise, she could always force linear regret in $T$ ~\citep{cover1965behavior, shalev2012online}. In this paper, we instead consider the continuous analog of prediction from experts, termed \emph{online convex optimization}. Here, $\cD$ is a continuum of possible actions, but the entirety of $f_t(\cdot)$ is revealed after it has been played. Surprisingly, the continuity of $\cD$ means that a player can guarantee sublinear regret against a strongly adaptive adversary, i.e., one who may choose $f_t$ even after observing $x_t$. %

 In many real-world applications, however, we desire an online algorithm to have greater continuity in its actions over the course of many rounds. In caching, for example, erratic online decisions may induce cache misses, and thus costly memory access procedures \citep{blasco2014multi}. More explicitly, the number of times that the player can switch her action between rounds may be strictly constrained. For example, suppose that the player makes 
predictions based on expert advice. If she would like to hire a new expert, 
she has to terminate the current contract, pay an early termination 
fee, and hire and pay a new expert. If hiring 
a new expert costs \$1000  in total and her budget is \$10000 
dollars, the number of her switches must be less than 10. This setting is 
called \emph{switching-constrained} or \emph{switching-budgeted} online 
learning~\citep{altschuler18a}. In these settings, it is necessary to assume an oblivious adversary: an adaptive adversary can force an algorithm with fewer than $K$ switches to incur linear regret by assigning 0 to a switched action between rounds, and 1 to a repeated action \citep{altschuler18a}. 

 Previous work has established the minimax regret of the switching-constrained multi-armed bandit and prediction from experts problems, but the minimax rate of switching-constrained online convex optimization was neither known nor algorithmically achieved. In this paper, we establish the minimax regret of \emph{switching-constrained} online 
convex 
optimization (OCO) against the strongest adaptive adversary, and in doing so, present a simple mini-batching algorithm that achieves this optimal rate.

We assume that $ \cD $, the action set of the player, is a compact, 
convex subset of $ \bR^n $. Let $ \cF $ be a family of differentiable convex 
functions from $ 
\cD $ 
to $ \bR $ from which the adversary selects each round's loss function, $f_t$. In the full-information setting (OCO), we 
assume that the player observes the loss function $ f_t $ after the adversary 
decides on it. %
The key ingredient, differentiating our setting from typical OCO, is a limit on the player's number of switches. Formally, 
given a sequence of points $ x_1,\dots,x_T $, let $ c(x_1,\dots,x_T) = 
\sum_{i=1}^{T-1} 
\ind[x_{i+1}\ne x_i] $ denote the number of switches. 
The player's action sequence must satisfy $ 
c(x_1,\dots,x_T)< K $ for some natural number $ K $.\footnote{In a $T$-round game, the maximum number of switches is always smaller than $T$. As a result, if $K>T$, the game becomes switching-unconstrained. Therefore, we assume throughout this paper that $K\le T$.} %

Given the player's action sequence $ x_1,\dots,x_T $ and the adversary's loss 
sequence $ f_1,\dots,f_T $, the usual regret is defined by the total accumulated loss 
incurred by the player, minus the total loss of the best possible single action 
in hindsight. We add an additional term and an outermost supremum that drives 
the regret of any player's sequence that violates the switch limit to infinity: 
\[ \mathcal{R}(\{x_i\}, \{f_i\}) = 
\sup_{\lambda>0} \left(  \sum_{i=1}^T f_i(x_i)  
-\inf_{x\in \cD}\sum_{i=1}^T 
f_i(x) + \lambda \ind[c(x_1,\dots,x_T) \ge K ] \right)\,,
 \]
where $\ind[\cdot]$ is the statement function whose value is $1$ if the proposition inside the brackets holds and is $0$ otherwise. In the following sections, we denote the switching-constrained minimax regret by $$\mathcal{R}(T, K) =  \inf_{x_1}\sup_{f_1}\dots\inf_{x_T}\sup_{f_T}\mathcal{R}(\{x_i\}, \{f_i\}),$$ where it will be clear from context from which sets $x_i$ and $f_i$ may be drawn.

\section{Related Work}\label{sec_relatedwork}

The framework of online convex optimization (OCO) and online gradient
descent 
were introduced by 
\citet{zinkevich2003online}. \citet{abernethy2008optimal} showed that the 
minimax regret of OCO against a strong adversary is $ \Theta(\sqrt{T}) $. \citet{abernethy2009stochastic} 
provided a geometric interpretation, demonstrating that the optimal regret can be viewed 
as the Jensen gap of a concave functional, and \citet{mcmahan2013minimax} studied 
the minimax behavior of unconstrained online linear optimization (OLO).

There is a substantial body of literature for switching-constrained and switching-cost online learning in the discrete settings, MAB and PFE. Switching-cost learning forces the player to pay for each switch  rather than enforcing a strict upper bound. Assuming potentially unbounded loss per round, \citet{cesa2013online} first showed 
that the minimax optimal rates of PFE and MAB with 
switching costs are $ \Theta(\sqrt{T}) $ and $ \tilde{\Theta}(T^{2/3}) $, 
respectively. \citet{dekel2014bandits} proved that the minimax regret of MAB 
with a unit switch cost in the standard setup (losses are bounded in $ [0,1] $) 
is $ \tilde{\Theta}(T^{2/3}) $.
\citet{devroye2013prediction} proposed a PFE algorithm whose expected regret 
and expected number of switches are both $ O(\sqrt{T\log n}) $, where $ n $ is 
the size of the action set. Finally, \citet{altschuler18a} showed that there is a phase transition, with respect to $K$, in 
switching-constrained PFE. If  
the maximum number of switches $ K $ is $ O(\sqrt{T\log n}) $ (low-switching 
regime), the optimal rate is $ \min\{ \tilde{\Theta}(\frac{T\log n}{K}), T \} 
$. If $ K $ is $ \Omega(\sqrt{T\log n}) $ (high-switching regime), the optimal 
rate is $ 
\tilde{\Theta}(\sqrt{T\log n}) $.  Once at least $\sqrt{T\log n}$ switches are permitted, the minimax regret surprisingly is not improved at all by allowing even more switches. In contrast to PFE, switching-constrained MAB 
exhibits no phase transition and the minimax rate is $ 
\min\{\tilde{\Theta}(\frac{T\sqrt{n}}{\sqrt{K}}),T\} $. 

Within the switching-constrained literature for continuous OCO, most directly comparable to our setting is 
\citet{jaghargh2019consistent}, which proposed a Poisson process-based OCO algorithm. The algorithm's expected regret is $ O(\frac{T^{3/2}}{\bE[K]}) $, where $\bE[K]$ may be set to any value provided that
$\bE[K]=\Omega(\sqrt{T})$. %
The expected regret, as a function of the expected number of switches, is suboptimal relative to the switching-constrained minimax rate we prove and achieve in this work. 

In the related learning with memory paradigm, the loss function for each round depends on the $M$ most recent actions. Switching-cost OCO can be viewed as a special case of learning with memory, in which $M=1$ and the loss functions are $g(x_t,x_{t-1})=f(x_t)+c\ind[x_t\neq x_{t-1}]$. \citet{merhav2002memory} introduced the concept of learning with memory, and used a blocking technique to achieve $O(T^{2/3})$ policy regret (a modification of standard regret for adaptive adversaries) and $O(T^{1/3})$ switches against an adaptive adversary. \citet{arora2012online} formally clarified and expanded the notion of policy regret for learning with memory, and presented a generalized mini-batching technique (applied here to achieve the matching upper bound) for online bandit learning against an adaptive adversary, converting arbitrary low regret algorithms to low policy regret algorithms. 
In \cref{additionalrelatedwork}, we also briefly discuss metrical task systems and online optimization with normed switching costs, but the main focus of this paper is the switching-constrained setting. %
\section{Contributions} \label{sec:contributions}
In this paper, we show that if both the player and the adversary select from the $L_2$ ball (i.e., $||x_t||_2 \leq 1$ and $f_t(x_t) = w_t \cdot x_t$ with $||w_t||_2 \leq 1$), then the minimax regret $\mathcal{R}(K, T)$ of switching-constrained 
online linear optimization is $ \Theta(\frac{T}{\sqrt{K}}) $. The precise bounds are contained below.

\begin{theorem}[Minimax regret of OLO] \label{thm:main}
The minimax regret of switching-constrained online linear optimization  
	satisfies the following bounds:
	
		\begin{compactenum}[(a)]
		\item $ \frac{T}{\sqrt{2K}} \leq \mathcal{R}(K, T) \leq \left\lceil\frac{T}{K}\right\rceil 
\sqrt{\frac{2(K+1)}{\pi}}
\le 
2\sqrt{\frac{2}{\pi}}\frac{T}{\sqrt{K}} $ for $n = 1$;
		\label{it:1-d}

\label{it:upper-bound}
		\item  $\frac{T}{\sqrt{K}} \leq  \mathcal{R}(K, T) \leq  \left\lceil\frac{T}{K}\right\rceil 
\sqrt{K} \le \frac{2T}{\sqrt{K}}$ for all $ 
		n>1 $. %
		
	\end{compactenum}
\end{theorem}
In \cref{subsec_lb_highd}, we prove the lower bound for dimension $n$ greater than 1, and in \cref{subsec_lb_1d}, we prove a one-dimensional lower bound that is slightly weaker than that of part (a), with prefactor $\frac{1}{2}$. To obtain the one-dimensional lower bound with tight constant $\frac{1}{\sqrt{2}}$, in \cref{sec_exactminimax} we analyze a carefully chosen OCO relaxation termed the ``fugal game". For ease of presentation, we provide an overview of the key intuitions of the fugal game and defer the complete analysis to the Supplementary Materials. 

Whereas it was sufficient to assume the adversary chose linear functions in all of the lower bounds, the upper bounds above are in fact derived from a far more general class of convex functions, as shown in the following proposition.

\begin{restatable}{prop}{propupperboundocohighd} \label{prop:upper_bound_oco_highd}
	If $ \cD $ is a convex and compact set from which the player draws $x_i$, and $ \cF $ is the family of 
	differentiable convex functions on $ \cD $, with uniformly bounded 
	gradient, from which the adversary chooses $f_i$, a mini-batching algorithm yields the upper bound
	$ \rtk \leq \lceil \frac{T}{K} \rceil 
	O(\sqrt{K}) = O(\frac{T}{\sqrt{K}}) $.
 \end{restatable}
 
Building on the previous proposition, the precise constants of the upper bounds in parts (a) and (b) for the case of linear functions are proven in the Supplementary Material (\cref{prop:exact_upper_bound_3d} and \cref{prop:upper-bound}). Combining the results of \cref{thm:main} and \cref{prop:upper_bound_oco_highd} immediately yields the following key minimax rate as a corollary.

 \begin{cor}\label{cor:theta-minimax-regret}
	The minimax regret of $ T $-round OCO with fewer than $K$ switches is $ \Theta(\frac{T}{\sqrt{K}}) $.  
\end{cor}
 
Before proceeding with the proofs of the preceding lower and upper bounds, we first consider a few implications and subtleties of the $O(\frac{T}{\sqrt{K}})$ minimax rate. First, note that if the player is not allowed to make any switch once she decides on her first action ($K=1$), the minimax regret $\Theta(\frac{T}{\sqrt{K}})=\Theta(T)$ is, naturally, linear. If the player is allowed to make more than $T-1$ switches,  $\Theta(\frac{T}{\sqrt{K}})=\Theta(\frac{T}{\sqrt{T}})=\Theta(\sqrt{T})$ recovers the classical $\Theta(\sqrt{T})$ regret of OCO~\citep{abernethy2008optimal}.
    
Furthermore, this minimax rate is in sharp contrast to switching-constrained PFE, the discrete  counterpart of switching-constrained OCO. 
	As noted in \cref{sec_relatedwork}, \citet{altschuler18a} proved a phase transition in 
	switching-constrained PFE between the high-switching and low-switching regimes. However, in the continuous full-information (OCO) setting, the minimax regret is the same regardless of the number of switches. \footnote{Note that the differing adversary strength is the direct cause of the disparity in phase transitions, but in an online learning problem, one usually assumes the strongest adversary that yields a sublinear minimax regret. The continuity of the (convex) action space circumvents the $\Omega(T)$ regret lower bound~\cite{altschuler18a},
	allowing sublinear minimax regret against an adaptive adversary, and as such it is the discrete vs. continuous action space that determines the most appropriate adversary.}  %

We also evaluate more closely the constants in the lower and upper bounds of \cref{thm:main} in the Supplementary Material, and mention the key results here. 

\begin{restatable}[The constant in $ \frac{T}{\sqrt{2K}} $ is 
unimprovable]{prop}{propunimprovableconstant}\label{prop:unimprovable-constant}
	The constant $ \frac{1}{\sqrt{2}} $ in the lower bound $ \rtk \ge \frac{T}{\sqrt{2K}} $ cannot be increased.
\end{restatable}

We prove this by considering the special case $K = 2$ (\cref{sub:tight}). 
	\footnote{To clarify, the proven rate $\rtk = \Theta(\frac{T}{\sqrt{K}})$ holds for any $K=o(T)$. For the special case of constant $K$,  the exact value of $c$ in $\rtk = c\frac{T}{\sqrt{K}}$ varies by $K$. We concretely compute several examples in \cref{exactvaluesofnormalizedminimaxregret} and then in \cref{sub:tight} show that for any universal lower bound $\rtk \geq c\frac{T}{\sqrt{K}}$, $c \geq \frac{1}{\sqrt{2}}$.} 
	
	In addition, the result of \cref{thm:main} exhibits a similar phenomenon to that 
	observed by \cite{mcmahan2013minimax} in the dimension-dependent minimax behavior of ordinary OCO. \citet{mcmahan2013minimax} 
	noted that the one-dimensional minimax value is approximately $ 0.8\sqrt{T} 
	$, while in higher dimensions (where both the player and the adversary select from the $n$-dimensional Euclidean ball) it is exactly $ \sqrt{T} $.  In 
	a switching-constrained OLO game, if the dimension is greater than 1, 
	\cref{thm:main} shows that the 
	minimax regret is asymptotically $ \frac{T}{\sqrt{K}} $ for all sufficiently 
	large $ T $. The following general proposition (proven in \cref{appendix:upper_bounds_OCO}) provides a link between the regret of the higher-dimensional and one-dimensional games:
	
	\begin{prop}\label{prop:non-decrease-in-n_forcontributions}
	The minimax regret 
	$\rtk$ is non-decreasing in the dimension $n$.
\end{prop}
	
 As a consequence, by part (a) of \cref{thm:main} and \cref{prop:non-decrease-in-n_forcontributions},    
	the one-dimensional minimax regret is asymptotically between 
	$\frac{T}{\sqrt{2K}}\approx 0.7\frac{T}{\sqrt{K}}$ and 
	$\frac{T}{\sqrt{K}}$ for all sufficiently large $T$. In particular, if $K>1$, we establish further in Appendix~\ref{appendix:upper_bounds_OCO} (\cref{prop:upper-bound}) that it
is at most $  0.87 
 \frac{T}{\sqrt{K}}  
$ for all sufficiently large $T$. This $ 0.87 
 \frac{T}{\sqrt{K}}  $ upper bound is strictly less than the higher-dimensional rate $\frac{T}{\sqrt{K}}$. 
	Thus the constant in the one-dimensional case is distinct from that of all 
	higher dimensions.

\section{Lower Bound} \label{sec_lowerbound}
We separately consider the high-dimensional (\cref{subsec_lb_highd}) and one-dimensional (\cref{subsec_lb_1d}) lower bounds, presenting distinct adversarial strategies for each.
\subsection{Dimension $n > 1$} \label{subsec_lb_highd}
In this section, we present the primary lower bound result for higher dimensions; a dimension-dependent result, in which player and adversary select from the $L_\infty$ rather than $L_2$ ball, can be found in \cref{appendix:sec:high-d}. We call the first round and every round in which the player chooses a new 
point a \emph{moving} round, and all other rounds \emph{stationary} rounds.%

The adversary's strategy attaining this lower bound is to follow a switching pattern identical to the player's, and to select a point via the orthogonal trick originally introduced in \cite{abernethy2008optimal}. 
It was non-trivial to adapt the orthogonal trick to the switching-constrained setting. First,
to prove our result, we had to impose a certain switching pattern upon the adversary which was not obvious \emph{a priori}, since the adversary is free to play any functions they wish from round to round. Without constraining the adversary to follow the player's switching pattern, the orthogonal trick cannot be applied. In addition, we found that this adversary's strategy can be adapted for $n=2$ --- thereby avoiding the need for special treatment as in $n=1$ --- via a geometric fact about the intersection of two closed half-spaces. \cite{abernethy2008optimal} only covered $n>2$. 

By the straightforward calculation $-\inf_{||x||_2 \leq 1} \sum_{i=1}^T w_i \cdot x = \sup_{||x||_2 \leq 1} \sum_{i=1}^T w_i \cdot x = \left\| \sum_{i=1}^T w_i \right\|,$ the regret $\rtk$ has two terms, $\sum_{i=1}^T w_i \cdot x_i$ and $\left\| \sum_{i=1}^T w_i \right\|$. The adversary wishes to make both terms non-decreasing over time. At the $i$-th round, if it is a moving round of the player, the adversary can choose a point $x_i$ whose inner products with $w_i$ and $\sum_{j<i}w_j$ are both non-negative. If it is a stationary round, the adversary selects her previous point. 

\begin{restatable}[Lower bound for higher dimensions]{prop}{thmhighd}\label{thm:high-d}
    The minimax regret 
	$ \rtk $ is at least $ \frac{T}{\sqrt{K}} $ for all $ 
	n>1 $. 
\end{restatable}

\begin{proof}
	Let $ 1=m_1 < m_2 <\cdots < m_K $ denote all moving rounds, with $ m_{K+1} = 
	T+1 $ and $ M_i = m_{i+1}-m_i $ denoting the length between two consecutive 
	moving rounds. Also, for  any integer $ 1\le t\le T $, let $ \pi(t) $ 
	be 
	the unique integer such that $ m_{\pi(t)} \le t < m_{\pi(t)+1} $, and define $ W_t = \ind[t>1]  \sum_{j=1}^{t} w_j $. 
	Let us consider 
	this adversary's strategy. 
	At the $ t 
	$-th round, if $ t $ is a moving round, the adversary chooses $ w_t $ such that $ \|w_t\|= 1 $, $ w_t \cdot x_t \ge 0  $, and $ w_t\cdot W_{t-1} \ge 0 $. 
	
	Such a vector $ w_t $ exists provided that the dimension $ n\ge 2 $. For 
	$n>2$, the subspace of $\mathbb{R}^n$ such that the latter two conditions 
	are 
	tight is of dimension $n-2 \geq 1$ and we may choose $w_t$ from this 
	subspace. For $n=2$, the latter two conditions 
	each define a closed halfspace 
	of 
	$\mathbb{R}^2$ and thus must have a non-empty intersection. If $ t 
	$ is a stationary round, the adversary chooses $ w_{m_{\pi(t)} } $, \ie, 
	the 
	same 
	vector that she plays at the moving around.
	Then the regret becomes \[
	\sum_{t=1}^{T} w_t\cdot x_t - \inf_{x\in B^n_2} \sum_{t=1}^{T} 
	w_t\cdot x 
	\ge - \inf_{x\in B^n_2} W_T \cdot x
	= \left\| W_T \right\|. \,
	\]
Now let us lower bound $ \| W_T \| $. 
	By the choice of $ w_{m_i} $, $ w_{m_i} $ is 
	perpendicular to $ \sum_{j=1}^{i-1} M_j w_{m_j} $. By iterating this relation, we obtain $$ \left\| \sum_{i=1}^{K} M_i w_{m_i}  \right\|^2 = \sum_{i=1}^{K}||M_i w_{m_i}||^2 = \sum_{i=1}^{K}M_i^2|| w_{m_i}||^2\,.$$
    It is thus the case that \[
	\| W_T \| = \left\| \sum_{i=1}^{K} M_iw_{m_i}  \right\|
	\ge \sqrt{ \sum_{i=1}^{K}M_i^2 \| w_{m_i} \|^2   }
	= \sqrt{ \sum_{i=1}^{K} 
		M_i^2 } \ge \frac{T}{\sqrt{K}}
	\,,
	\]
	where the last inequality follows from the Cauchy-Schwarz inequality.
\end{proof}

\subsection{Dimension $n = 1$} \label{subsec_lb_1d}
For the one-dimensional case, we present a strategy for the adversary that guarantees regret at least $\frac{T}{2\sqrt{K}}$, the correct order of minimax regret. However, note that the prefactor $\frac{1}{2}$ is suboptimal, and in \cref{sec_exactminimax} we elaborate on a more involved analysis that will yield a tighter constant.

\begin{prop}[Lower bound for one dimension]\label{prop:1d-lowerbound}
The minimax regret $\rtk$ is at least 
$ \frac{T}{2\sqrt{K}}$ for $n=1$.
\end{prop}

\begin{proof}
Let $W_t=\sum_{i=0}^{t-1} w_i$. Recall from \cref{subsec_lb_highd} that the regret consists of two terms, $\sum_i x_iw_i$ and $\left | \sum_i w_i \right|$. At a high level, the adversary either chooses to maximize the first term, or determines that the contribution of the second term is high enough and prevents the regret from changing further by selecting the $0$ loss function for the remainder of the rounds. Concretely, let $w_t=0$ if $|W_t|\geq T/\sqrt{K}$. We refer to this as the ``stopping condition". When $|W_t|<T/\sqrt{K}$, let $w_t=1$ if $x_t\geq- {W_t\sqrt{K}/{T}}$, and $w_t=-1$ otherwise. %

Let $t_i$ denote the round of the $i$th switch in $x$, with $t_0=0$, and let $t_K$ denote the round where $|W_t|$ reaches $\frac{T}{\sqrt{K}}$, or $T$ if it does not  exist.\footnote{If the player chooses to use $S < K-1$ switches, it suffices to assign $\{t_i\}_{i=0}^K$ by arbitrarily splitting up blocks in which the player plays a fixed action, as if she had switched. For example, if the player never switches but $K > 1$, $\{t_i\}$ can be any strictly increasing sequence of round indices.} Let $T_i=t_{i+1}-t_{i}$ then denote the length of the $(i+1)$st block, and note that $w_t$ remains fixed within each interval, i.e. the adversary copies the player's switching pattern. We then have
$$ \rtk =\sum_{t<T} x_t w_t +|W_{t_K}|
    = \sum_{i=0}^{K-1} (x_{t_i}+W_{t_i}\sqrt{K}/{T}) w_{t_i}T_i- \sum_{i=0}^{K-1} (W_{t_i}\sqrt{K}/{T}) w_{t_i}T_i +|W_{t_K}| \,.$$
By the adversary's strategy, $(x_{t_i} + W_{t_i}\sqrt{K}/{T})w_{t_i}$ is always non-negative, and so this expression is at most $- \sum_{i=0}^{K-1} (W_{t_i}\sqrt{K}/{T}) w_{t_i}T_i +|W_{t_K}|$. Since $w_t$ within each interval is fixed, we have $w_{t_i}T_i=W_{t_{i+1}}-W_{t_i}$. Moreover, it suffices to consider only $i$ in the sum such that $w_i$ has not yet been set to $0$. Applying these facts, the expression becomes:
    \begin{align*}
     & -\sum_{i=0}^{K-1} \left(\frac{\sqrt{K}}{{2T}}\right)(W_{t_{i+1}}^2-W_{t_i}^2-(w_{t_i}T_i)^2) +|W_{t_K}|\\
     ={}& -\left(\frac{\sqrt{K}}{{2T}}\right)W_{t_{K}}^2+\sum_{i=0}^{K-1} \left(\frac{\sqrt{K}}{{2T}}\right)T_i^2  +|W_{t_K}|\,.
    \end{align*}
    Note that $\sum_{i=0}^{K-1}T_i=t_K$, so by Cauchy-Schwarz we can lower bound this expression by $-\left(\frac{\sqrt{K}}{{2T}}\right)W_{t_{K}}^2+ \left(\frac{1}{{2T\sqrt{K}}}\right)t_K^2 +|W_{t_K}|$.
    If the stopping condition is never met, then $t_K = T$ and the remaining terms are non-negative, since they factorize as $|W_{t_K}|(1 - \frac{\sqrt{K}}{2T}W_{t_K})$ and $|W_{t_K}| < \frac{T}{\sqrt{K}}$ by assumption. Otherwise, we have $|W_{t_K}|\in[\frac{T}{\sqrt{K}},\frac{T}{\sqrt{K}}+1]$, and thus $(|W_{t_K}|-\frac{T}{\sqrt{K}})(|W_{t_K}|-\frac{T}{\sqrt{K}}-1)\leq 0$. Consequently, the chain of lower bounds continues with:
    \begin{align*}
     \rtk 
    \geq{} \min\Bigg\{ \frac{T^2}{{2T}\sqrt{K}},  & -\left(\frac{\sqrt{K}}{{2T}}\right)W_{t_{K}}^2+ \frac{t_K^2}{{2T}\sqrt{K}} +|W_{t_K}|\\
    &+\left(\frac{\sqrt{K}}{{2T}}\right)\left(|W_{t_K}|-\frac{T}{\sqrt{K}}\right)\left(|W_{t_K}|-\frac{T}{\sqrt{K}}-1\right) \Bigg\} \,.
\end{align*}
Finally, this entire expression is lower bounded by $\frac{T}{2\sqrt{K}}$. This follows by noting that the second term equals $(\frac{1}{{2T}\sqrt{K}})t_K^2 +|W_{t_K}|(-\frac{\sqrt{K}}{{2T}})+(\frac{T}{\sqrt{K}}+1)/2$, which is greater than or equal to $$\left(\frac{1}{{2T}\sqrt{K}}\right)t_K^2
\left(\frac{T}{\sqrt{K}}+1\right)\left(-\frac{\sqrt{K}}{{2T}}\right)+\left(\frac{T}{\sqrt{K}}+1\right)/2=\left(\frac{1}{{2T}\sqrt{K}}\right)t_K^2 
-\frac{\sqrt{K}}{{2T}}+\frac{T}{2\sqrt{K}}\,.$$ Noting that the second term assumed the stopping condition was met and therefore $t_K\geq |W_{t_K}|\geq \frac{T}{\sqrt{K}} $, the resulting quantity is then lower bounded by  $(\frac{1}{{2T}\sqrt{K}})(\frac{T}{\sqrt{K}})^2 
-\frac{\sqrt{K}}{{2T}}+(\frac{T}{\sqrt{K}})/2\geq (\frac{T}{\sqrt{K}})/2$.
\end{proof}

\section{Upper Bound} \label{sec_upperbound}
In this section, we prove \cref{prop:upper_bound_oco_highd} and thereby derive an upper bound for switching-constrained OCO to match 
the lower bounds of \cref{subsec_lb_1d} and \cref{subsec_lb_highd}. We begin with a 
simple algorithm achieving the correct minimax regret, $O(\frac{T}{\sqrt{K}})$. In \cref{appendix:upper_bounds_OCO}, we expand on this result with a closer evaluation of the constant.

 	\begin{proof}
 First, we claim that the minimax regret $ \rtk $ is a 
 non-decreasing function in $ T $. To see this, consider the situation where we 
 have more rounds. The adversary can play $ 0 $ in all additional rounds and 
 this does not decrease the regret. Therefore, we obtain that $ \rtk \le \mathcal{R}(T_1, K) $, where $ T_1 = \lceil 
 \frac{T}{K} \rceil K \ge T $.

 To derive an  upper bound for $ \mathcal{R}(T_1, K)$, we mini-batch the $ T_1 $ 
 	rounds into $ K $ equisized epochs, each having size
 	$  \frac{T_1}{K} = \lceil \frac{T}{K} \rceil  $. Let $ E_i $ denote 
 	the set of all rounds that belong to the $ i $-th epoch. We have $ 
 	 E_i = \{ \frac{T_1}{K}(i-1)+1,\frac{T_1}{K}(i-1)+2,\dots,\frac{T_1}{K}i \} 
 	 $.
 	 The 
 	epoch loss of the $ i $-th epoch is the average of loss functions in this 
 	epoch, \ie, $ \bar{f}_i\triangleq  \frac{1}{|E_i|} 
 	\sum_{j\in E_i} f_j $. If we run a minimax 
 	optimal algorithm for unconstrained OCO (for example, online gradient 
 	descent \citep{zinkevich2003online}) on the epoch losses $ 
 	\bar{f}_1,\dots,\bar{f}_{K} $ and obtain the player's action sequence $ 
 	\bar{x}_1,\dots,\bar{x}_{K} $, our strategy is to play $ \bar{x}_i $ at all 
 	rounds in the $ i $-th epoch. 
 	This method was originally discussed in \cite{arora2012online,dekel2012optimal}. Using this 
 	mini-batching method, we deduce that the regret is upper bounded by $ 
 	\frac{T_1}{K} O(\sqrt{K})  = \lceil \frac{T}{K} \rceil O(\sqrt{K}) = 
 	O(\frac{T}{\sqrt{K}}) $, where 
 	$O(\sqrt{K})$ is the standard upper bound of the regret of a $K$-round OCO 
 	game. 
	\end{proof}
\section{Exact Minimax Bound} \label{sec_exactminimax}

In \cref{subsec_lb_1d}, we analyzed an adversarial strategy for OLO that gave the desired order of lower bound, $O(\frac{T}{\sqrt{K}})$. However, a  separate (and far more involved) analysis is needed to obtain the tight constant in $\frac{T}{\sqrt{2K}}$, which revealed the surprising parallels (in the differing constants of one and higher dimensions) between switching constrained and unconstrained OCO that were stated in \cref{sec:contributions}. In this section, we outline the key ideas of our approach to obtaining a \emph{tight} lower bound. The analysis to complete the proof of this result is in \cref{sec:1-d}.

The reader may wonder why a simpler argument cannot prove the tight one-dimensional lower bound. To shed some light on this difficulty, we note that the truly minimax optimal strategy does not necessarily involve the player choosing to switch at uniform points throughout the $T$ rounds. In \cref{unequal_blocks} of the Supplementary Material, we consider the case where $K=3$ and prove that the first switch happens at round approximately $0.29T$, rather than $0.33T$. Although the mini-batching algorithm of our upper bound makes uniformly spaced switches, this simple fact demonstrates that the resultant constant cannot be tight, and that such an assumption --- simplifying though it may be --- would not be valid for the lower bound. 

\subsection{Fugal Game}
Our minimax analysis for the one-dimensional game (OCO) is through what we call 
\emph{fugal games}. In a fugal game, 
the adversary is weakened by being 
constrained to adhere to the player's switching pattern, and to only select from $ 
-1 $ and $ 1 $. Furthermore, the horizon $ T $ and the interval between 
consecutive switches are allowed to take non-negative real values; one way of interpreting this is that we allow the player to switch in the ``middle'' of a round, provided they still respect the switching budget overall and play for a total of $T$ complete rounds. Note that, because the adversary is forced to maintain the same action until the player switches,
allowing non-negative real-valued interval between consecutive switches 
only strengthens the player. \textbf{The minimax value of fugal games thus provides a lower bound for the 
minimax value of switching-constrained OCO.}

As motivation for our terminology, recall that at the exposition of a fugue, one voice begins by introducing a particular melodic theme. Afterward, a new voice repeats the same melody for the same duration, but transposed to a new key. This may repeat multiple times as subsequent voices alternate between the introduction of a new melody (sometimes termed the ``question”), and its transposed repetition (the ``answer”). In the original switching-constrained OCO framework, the adversary is under no obligation to repeat her loss function for the same number of rounds as the player sticks to the same action. However, if we restrict the adversary to copy the switching pattern of the player, their interaction becomes reminiscent of a fugal exposition. The player begins by choosing a key ($x_i$) for her melody, and the adversary necessarily responds at a new pitch ($f_i$); optionally based on $f_i$, the player chooses the duration $M_i$ to maintain pitch $x_i$, and the adversary imitates her by playing $f_i$ for length $M_i$ as well. This repeats until all $K$ question-and-answer pairs are done. Thus, we call this relaxation of OCO the \emph{fugal game}.

\begin{prop}[Asymptotic tightness of fugal game; informal]\label{informalprop_asymptotic-tightness}
The minimax regret of the fugal game is \emph{equal} to the normalized minimax regret of switching-constrained online convex optimization asymptotically, when $K$ is constant and $T$ approaches infinity.
\end{prop}

The proof of this proposition can be found in the Supplementary Material (\cref{prop:asymptotic-tightness}), and proceeds by devising a strategy for the player in OCO based on the fugal game. \cref{informalprop_asymptotic-tightness} justifies that the lower bound of the fugal game's minimax regret is the ``right" quantity to study, as it captures the correct asymptotic behavior.

\paragraph{Proof Overview.}

We solve the minimax behavior 
of fugal games by studying a generalization of their minimax regret function, with 
an initial bias. %
This generalization is called a \emph{fugal 
function}. We first derive the recursive relation of the fugal function, and 
then show that the fugal function is at least %
the 
absolute value of the initial bias. %
To average out the influence of $ T $, we 
define the \emph{normalized minimax regret} and show that it is indeed 
independent of $ T $. The normalized minimax regret inherits a recursive 
relation from the fugal function. However, it is mathematically challenging to 
solve the exact values of the normalized minimax regret. 
In light of this, we consider an alternative quadratic lower 
bound whose recursive relation can be solved in closed form, although significant 
technical effort is required. This constitutes the most computationally ``hardcore'' section of 
this paper. Our minimax analysis for the one-dimensional game follows immediately from the quadratic lower bound.

To aid in the recursive analysis, we generalize ordinary minimax regret slightly by introducing an initial bias. Let $I = [-1,1]$. 
Formally, the minimax regret with $T$ rounds, a maximum number of $ k $ 
switches, and an initial bias $ Z $ is defined by
\begin{equation*}
 R_k(T, Z)
 = \inf_{x_1\in I} \sup_{w_1\in I} \dots \inf_{x_T\in 
I}\sup_{w_T\in I} \sup_{\lambda>0} 
 \left(  \sum_{i=1}^T w_i x_i + 
\left|Z + 
\sum_{i=1}^T 
w_i\right| + \lambda \ind[c(x_1,\dots,x_T)\ge k ] \right)\,.
 \end{equation*}
 We motivate the initial bias $Z$ as follows. 
 When the adversary tries to maximize regret in any given round, her choice is determined by the tradeoff between maximizing the first term and maximizing the second term above. To focus wholly on the first term, the adversary could specifically penalize the player's last action by playing $w_t=\sign(x_t)$. To focus wholly on the second term, the adversary could instead amplify the term within the absolute value by playing $w_t=\sign(\sum_{i=1}^{t-1}w_i)$. At each round, the adversary chooses a value to optimize this tradeoff given the results of previous rounds. When setting up recursive relations between $R_k$ and $R_{k+1}$, the first term decouples neatly by round, but the second term does not. Thus, an initial bias term is necessary for deriving a recursive relation, as a sort of state that is passed between $R_k$'s. Extending this definition to the fugal game yields the \emph{fugal function}
 \begin{align*}
  r_k(T, Z)
={}
\inf_{x_1\in I} \max_{w_1=\pm 1} & \inf_{M_1\ge 0}\dots 
\inf_{x_k\in I} \max_{w_k=\pm 1} \inf_{M_k\ge 0} \sup_{\lambda>0} \\
& \left(  
\sum_{i=1}^k 
M_iw_i x_i  + \left|Z + 
\sum_{i=1}^k 
M_iw_i\right| + \lambda \ind[\sum_{i=1}^k M_i\ne T] \right)\,,
\end{align*}
 where $ M_i $ is the length between two moving rounds, and we have relaxed $ M_i $ by 
 allowing it to take any non-negative real values. 
 
 We can normalize out the influence of $T$ by setting $u_k(z) = \frac{r_k(T,zT)}{T}$, thereby reducing our task to the analysis of the single-variate function $u_k(z)$. In this way, the fugal game decouples the minimax regret from the discrete nature of $T$. Central to the recursive analysis of $u_k(z)$ is the ``fugal operator'':
 \begin{defn}[Fugal operator]
    The \emph{fugal operator} transforms continous functions, and is defined by \[ 
(\cT f)(z) \triangleq \inf_{x\in [-1,1]} \max_{w=\pm 1}  
\inf_{\substack{|z'|<1\\ w(z'-z)\ge 0}} 
\frac{(1+wz) f(z')+ x 
	(z'-z)}{1+z'w} 
\,.
 \]
\end{defn}
 It turns out that $u_k$ satisfies the concise recursive relation $ u_{k+1} = \cT u_k$. By closely analyzing this relation (the full details of which are contained in \cref{sec:1-d} of the Supplementary Material), we eventually lower bound $u_k(0)$ to obtain %
 a sharp one-dimensional lower bound, $\frac{T}{\sqrt{2K}}$.

\section{Conclusion}
In this work, we considered switching-constrained online convex optimization, a setting which until now had received comparatively little attention relative to switching-constrained multi-armed bandits and prediction from experts. In the OCO setting, we established the minimax regret against the strongest adaptive adversary as $\Theta(\frac{T}{\sqrt{K}})$. 
 The upper bound on minimax regret was constructive, using the mini-batching paradigm to obtain a  meta-algorithm for achieving the correct minimax rate. This effectively solves the question of optimal algorithms for switching-constrained online convex optimization.

\section*{Broader Impacts}
In this paper, we fully charachterize the minimax regret of switching-constrained online convex optimization. 
Since it is a theoretical result in nature, the broader impact discussion is not applicable.

\section*{Acknowledgments}
LC and QY were supported by Google PhD Fellowship. HL was supported by the Fannie and John Hertz Foundation. AK was partially supported by NSF (IIS-1845032), ONR (N00014-19-1-2406), AFOSR (FA9550-18-1-0160), and TATA Sons Private Limited.
We would like to thank Jacob Abernethy, Hossein Esfandiari, Peng Zhang, and Peilin Zhong for
helpful conversations during the early stages of this work.

\bibliographystyle{plainnat}
\bibliography{reference}

\clearpage
\appendix
\part*{Supplementary Material}
\addcontentsline{toc}{part}{Supplementary Material}

\section{Additional Related Work} \label{additionalrelatedwork}

Metrical task systems is another broad area that overlaps switching-cost OCO, in which the goal is to minimize both movement and cost per round. However, it fundamentally departs from OCO in that the adversary reveals the loss function \emph{first} in each round, and that success is measured by competitive ratio rather than regret. \citet{andrew2013metrics} considered OCO with a seminorm switching penalty added to the regret, and bridged these two modes by demonstrating that no algorithm can simultaneously achieve both sublinear regret and constant competitive ratio. Along the way, they also showed that gradient descent achieves $O(\sqrt T)$ regret, even with added seminorm switching costs. (Note, however, that a binary penalization for switching is not a seminorm.)

We here mention assorted results in switching-cost or switching-constrained OCO, although they differ
significantly in conventions. Instead of a binary penalization for switching per round, \citet{li2018predictionwindow} added a quadratic switching cost to the regret, $\sum_{i=1}^{T-1}\|x_{i+1}-x_i\|^2$. However, they allowed the player some clairvoyance about future loss functions in the form of a fixed ``lookahead'' window, and thus consider a modified ``dynamic'' regret. 
\citet{badiei2015ramp} similarly considered a non-binary switching penalization and a finite lookahead window, with a hard constraint on the total $L_1$ distance between consecutive actions and performance evaluation in terms of the competitive ratio. \citet{gofer2014sc} demonstrated that for OCO with linear objectives and any normed switching cost, no algorithm can achieve bounded regret against every loss sequence with a finite quadratic variation. \citet{anava2015memory} presented algorithms for OCO with memory against an \emph{oblivious} adversary, achieving $\tilde{O}(T^{1/2})$  regret and $\tilde{O}(T^{1/2})$ binary switching cost. This result demonstrated that restricting the adversary can lead to regret-switching dependencies stronger than we prove are optimal against an \emph{adaptive} adversary.
\citet{awerbuch1996making} analyzed limited switching from a multiplicative ratio perspective.

\section{Preliminaries}\label{preliminaries}
We denote the $p$-norm by $\|\cdot\|_p$. If $x$ and $y$ are two vectors living in $\bR^n$, we write $x\cdot y$ for their inner product. If $x_i$ is a vector, let $x_{i,j}$ denote its $j$-th coordinate. Let $\ind[\cdot]$ denote the indicator function whose value is $1$ if the statement inside the brackets holds and is $0$ otherwise.

In the special case of switching-constrained online convex optimization that we 
focus on, the regret 
is given 
by 
\begin{equation} \label{eq:regret-p-norm}
\sum_{i=1}^T w_i x_i 
-\inf_{\|x\|_p\le 1}
\sum_{i=1}^T 
w_i\cdot x 
=
\sum_{i=1}^T w_i x_i 
+\sup_{\|x\|_p\le 1}
\sum_{i=1}^T 
w_i\cdot (x) 
\stackrel{(a)}{=}
\sum_{i=1}^T w_i x_i 
+\left\|
\sum_{i=1}^T 
w_i\right\|_{p/(p-1)} 
\,. 
\end{equation}
where $(a)$ is because $\|\cdot\|_{p/(p-1)}$ is the dual norm of $\|\cdot\|_p$. 
Recall that if $ \|\cdot \| $ is a norm, its dual norm $ \|\cdot \|_* $ is 
defined by $ \|z\|_*\triangleq \sup_{\|x\|\le 1} z\cdot x $. 
Let $ B_p^n \triangleq \{ x\in \bR^n : \|x\|_p\le 
1 \} $ be the $n$-dimensional unit ball and let $ 
B^{*n}_q=\{ 
f(x)=w\cdot 
x : w\in \bR^n, \|w\|_q\le 1 \} $ denote the dual unit ball,
where $ 1\le p,q\le \infty $. 
Since all $ p 
$-norms coincide if $ n=1 $, we simply write $ B^1 $ and $ B^{*1} $ and do not 
specify an explicit $ p $ and $ q $. 
We also use the more detailed notation $\roco(\mathcal{D}, \mathcal{F}, K, T)$ to indicate that the player chooses actions from $\mathcal{D}$, the adversary chooses functions from class $\mathcal{F}$, and there are $T$ rounds with fewer than $K$ switches.

We use the abbreviations OLO for online linear optimization, OCO for online convex optimization, BLO for bandit linear optimization, and BCO for bandit convex optimization.

\paragraph{Moving and Stationary Rounds}
We call the first round and every round in which the player chooses a new 
point a \emph{moving} round. Formally the $ i $-th round is a moving round if $ 
i=1 $ or $ x_i\ne x_{i-1} $. 
We call every round in which the player 
sticks to her previous point a \emph{stationary} round.

\section{Lower Bound for One-Dimensional Switching-Constrained OCO}\label{sec:1-d}

In this section, we will show the $\frac{T}{\sqrt{2K}}$ lower bound for the  
minimax regret of the one-dimensional game. 

\begin{prop}[Lower bound for one-dimensional game]\label{thm:1-d}
	The minimax regret 
	$ \roco(B^1, B^{*1}, K, T) $  
	is at least $ \frac{T}{\sqrt{2K}} $.
\end{prop}

The proof of \cref{thm:1-d} can be found in 
\cref{sub:quadratic-lower-bound}. However, the proof relies on results in all 
preceding subsections.

Recall that we defined the minimax regret with $T$ rounds, a maximum number of $ k $ 
switches, and an initial bias $ Z $ is by
\begin{align*}
 R_k(T, Z)
 ={} \inf_{x_1\in [-1,1]} \sup_{w_1\in [-1,1]} & \dots \inf_{x_T\in 
[-1,1]}\sup_{w_T\in [-1,1]} \sup_{\lambda>0}\\
& \left(  \sum_{i=1}^T w_i x_i + 
\left|Z + 
\sum_{i=1}^T 
w_i\right| + \lambda \ind[c(x_1,\dots,x_T)\ge k ] \right)\,.
 \end{align*}

 \subsection{Lower Bound via Fugal Games}

As in the main body of the paper,
we consider the modified \emph{fugal game}, with the minimax regret of a fugal game with $ T $ rounds ($ T\in \bR_{\geq0} 
$), a maximum 
number 
of $ k-1 $ switches, and an initial bias $ Z $ defined by \begin{equation} 
\label{eq:def_rk}
\begin{split}
 r_k(T, Z) ={}
  \inf_{x_1\in [-1,1]} \max_{w_1=\pm 1} & \inf_{M_1\ge 0}\dots 
\inf_{x_k\in [-1,1]} \max_{w_k=\pm 1} \inf_{M_k\ge 0} \sup_{\lambda>0}\\ &\left(  
\sum_{i=1}^k 
M_iw_i x_i + \left|Z + 
\sum_{i=1}^k 
M_iw_i\right| + \lambda \ind[\sum_{i=1}^k M_i\ne T] \right)\,,
\end{split}
 \end{equation}
 where $ M_i $ is the length between two moving rounds, and we have relaxed $ M_i $ by 
 allowing it to take any non-negative real values. The function $ r_k(T,Z) $ 
 is a \emph{fugal function}.

 The minimax regret in a 
 fugal game 
 gives 
 a 
 lower bound for the minimax regret of interest.\footnote{Intuitively, strengthening the player and weakening the adversary can only lower the minimax regret. Formally, this holds as a result of a few straightforward facts. If $X' \subseteq X$ and $Y' \subseteq Y$, then $\inf_{x \in X'}\sup_{y \in Y}f(x,y) \geq \inf_{x \in X}\sup_{y \in Y'}f(x,y)$. To apply this recursively, simply note that if for all $x \in X$ and $y \in Y$ we have $f(x,y) \leq g(x,y)$, then $\inf_{x \in X}\sup_{y \in Y}f(x,y) \leq \inf_{x \in X}\sup_{y \in Y}g(x,y)$. Thus restricting the range of possible values for each supremum term in $\roco(\cB^n, \cL^n, K, T)$ (corresponding to the adversary's choices), followed by enlarging the range of possible values for each infimum term (corresponding to the player's actions), only lowers the ultimate minimax regret.}
 In other words, it holds that 
 $ 
r_k(T,Z)\le R_k(T,Z) $ if $ T $ is a natural number.
Furthermore, whenever it is the player's turn to play, she must optimize over not only the action, but also the optimal length of time to \emph{maintain} that action to minimize her ultimate regret. As a result of this basic intuition, the function $ r_k(T,Z) $ satisfies the following recurrence 
relation for all $ k\ge 1 $
\begin{equation}\label{eq:r_iteration}
r_{k+1}(T,Z) = \inf_{x\in [-1,1]} \max_{w=\pm 1}  \inf_{0\le t\le 
T} 
\left( 
twx + {r}_k(T-t, Z+tw ) \right)  \,.
 \end{equation}

\subsection{Absolute Value Bounds for Fugal Games}
In this subsection, we derive basic properties of the fugal functions. 
\cref{lem:absolute-bound} shows that the function $ r_k(T,Z) $ is at least $ 
|Z| $ for all $ Z\in \bR $ and that the inequality is tight if $ |Z|\ge T $.
\begin{lemma}[Absolute value bounds for fugal games]\label{lem:absolute-bound}
	The minimax regret of a fugal game with an initial bias $ r_k(T,Z) $ 
	satisfies the following two properties
	\begin{compactenum}[(a)]
		\item $ r_k(T,Z)\ge |Z| $ for all $ Z\in \bR $; and\label{it:lb_abs}
		\item $ r_k(T,Z)=|Z| $ if $ |Z|\ge T $.\label{it:eq_abs}
	\end{compactenum}
\end{lemma}
\begin{proof}
	To prove part (\ref{it:lb_abs}), it suffices to design an adversary's 
	strategy that satisfies this lower bound. Suppose that the adversary always 
	plays $ \sign(Z) $, or 1 if $Z=0$. Since $ \sum_{i=1}^k M_i w_i x_i \ge -\sum_{i=1}^{k}M_i 
	= -T $, if $z \neq 0$ we have \[
	\begin{split}
	\sum_{i=1}^k 
	M_iw_i x_i + \left|Z + 
	\sum_{i=1}^k 
	M_iw_i\right| \ge{}& -T + \left|Z + T\sign(Z) \right| = -T + \left| (|Z|+T)\sign(z) \right| \\
	={}& -T + (|Z|+T) = |Z|\,.
	\end{split}
	 \]
	 When $Z=0$, the expression above is clearly at least $-T+T=0 = |Z|$. Therefore the lower bound $ r_k(T,Z)\ge |Z| $ holds.
	 
	To prove part (\ref{it:eq_abs}), we will show that $r_k(T,Z)\le |Z|$ if $ 
	|Z|\ge T $. First, we assume $ Z\ge T $. In this case, 
	we have $ Z + 
	\sum_{i=1}^k 
	M_iw_i\ge T-T=0 $. Since we are certain of the sign of the expression 
	inside the absolute value, we remove the absolute value and obtain 
	 \[ 
	\sum_{i=1}^k 
	M_iw_i x_i + \left|Z + 
	\sum_{i=1}^k 
	M_iw_i\right| = Z+ \sum_{i=1}^k M_iw_i(1+x_i)\,.
	 \]
	 The above expression equals $Z $ if the player always plays $ -1 $. 
	 Therefore, if $ Z\ge T $, we obtain $ r_k(T,Z)\le Z = |Z| $.
	 
	 If $ Z\le -T $, since $ Z + 
	 \sum_{i=1}^k 
	 M_iw_i\le -T+T = 0 $, we have
	 \[ 
	 \sum_{i=1}^k 
	 M_iw_i x_i + \left|Z + 
	 \sum_{i=1}^k 
	 M_iw_i\right| = \sum_{i=1}^k 
	 M_iw_i x_i - (Z + 
	 \sum_{i=1}^k 
	 M_iw_i) = -Z+\sum_{i=1}^{k}M_i w_i(x_i-1)\,.
	  \]
	  Again, the above expression equals $ -Z $ if the player always plays $ 1 
	  $. Therefore, if $ Z\le -T $, we get $ r_k(T,Z)\le -Z = |Z| $. In both 
	  cases, we show that $ r_k(T,Z)\le |Z| $, which completes the proof.
\end{proof}

\subsection{Extraspherical Minimax Regret}

Recall that in \eqref{eq:r_iteration}, the minimum is taken over all $ t $ 
between $ 0 $ and $ T $. Let us consider a subset of $ t $ such that $ 
|Z+tw|\le T-t $. Notice that if $ t<T $, ${|Z+tw|}$ belongs to a one-dimensional ball $[-T+t,T-t]$. 
\cref{lem:out-of-bound-minimax} gives a minimax lower bound for $ t $ such that 
$ {|Z+tw|} $ is outside the ball. Since it is shown in 
\cref{lem:absolute-bound} that $ r_k(T,Z)=|Z| $ if $ |Z|\ge T $, we assume in 
the following lemma that $ |Z|<T $.
\begin{lemma}[Extraspherical minimax regret]\label{lem:out-of-bound-minimax}
If $ Z\in (-T,T) $,
\[	 \inf_{x\in [-1,1]} \max_{w=\pm 1}  \inf_{\substack{0\le t\le 
		T\\ |Z+tw|\ge T-t}} 
	\left( 
	twx + {r}_k(T-t, Z+tw ) \right) = \frac{Z^2+T^2}{2T} \,.\]
\end{lemma}
\begin{proof}
Let us first expand the maximum operator
	\[ 
	\begin{split}
	A\triangleq & \inf_{x\in [-1,1]} \max_{w=\pm 1}  \inf_{\substack{0\le 
	t\le 
			T\\ |Z+tw|\ge T-t}} 
	\left( 
	twx + {r}_k(T-t, Z+tw ) \right)\\ 
	={} &  
	\inf_{x\in [-1,1]} \max\left\{
	  \inf_{\substack{0\le t\le 
			T\\ |Z+t|\ge T-t}} 
	\left( 
	tx + {r}_k(T-t, Z+t ) \right),
	\inf_{\substack{0\le t\le 
			T\\ |Z-t|\ge T-t}} 
	\left( 
	-tx + {r}_k(T-t, Z-t ) \right)
	\right\}\,.
	\end{split}
	 \]
As the first step, we study the situation where $ w=1 $.
If $ |Z+t|\ge T-t $, it implies $ Z+t\ge T-t $ or $ Z+t\le -(T-t) $. The 
second 
case 
is impossible since it is equivalent to $ Z\le -T $ (recall our assumption that 
$ |Z|<T $). The first case is equivalent 
to $ t\ge (T-Z)/2 $. Therefore the range of $ t $ over which the minimum is 
taken 
is $ \frac{T-Z}{2}\le t\le T $. 
The expression of which we take the infimum becomes 
  \[ 
 tx + {r}_k(T-t, Z+t ) 
 \stackrel{(a)}{=}tx+|Z+t|=t(x+1)+Z\,,
  \]
  where $ (a) $ is due to  \cref{lem:absolute-bound}
 by recalling $ 
  \left|{Z+t}\right|\ge {T-t} $.
  In this way, since $ 1+x\ge 0 $, we obtain a cleaner expression for the 
  innermost infimum in the 
  case $ w=1 $
  \[ 
  \inf_{\substack{0\le t\le 
  		T\\ |Z+t|\ge T-t}} 
  \left( 
  tx + {r}_k(T-t, Z+t ) \right) = 
  \inf_{\frac{T-Z}{2}\le t\le T} t(x+1)+Z = 
  \frac{T-Z}{2} (1+x)+Z\,.
   \]
   The second step is to study the situation where $ w=-1 $. If $ |Z-t|\ge T-t 
   $, 
   it implies $ Z-t\ge T-t $ or $ Z-t\le -(T-t) $. The first case is 
   impossible as 
   it is equivalent to $ z\ge 1 $. The second case is equivalent to $ 
   t\ge (T+Z)/2 
   $. Therefore the range of $ t $ over which the minimum is taken becomes 
   $ \frac{T+Z}{2}\le t\le T $. 
   The expression of which we take the infimum becomes
   \[ 
   -tx + {r}_k(T-t, Z-t )=-tx + |Z-t| = t(1-x)-Z\,.
    \]
   where we use \cref{lem:absolute-bound} again in the first equality. 
   Since $ 1-x\ge 0 $, we obtain a similar clean expression for the 
   innermost infimum in the case $ w=-1 $
   \[ 
   \inf_{\substack{0\le t\le 
   		T\\ |Z-t|\ge T-t}} 
   \left( 
   -tx + {r}_k(T-t, Z-t ) \right) = \inf_{\frac{T+Z}{2}\le t\le T} 
   t(1-x)-Z= 
   \frac{T+Z}{2}(1-x)-Z\,.
    \]
    Therefore, the extraspherical minimax can be lower bounded as follows \[
    \begin{split}
    A ={}& 
     \inf_{x\in [-1,1]}\max \left\{ \frac{T-Z}{2}(1+x)+Z,
     \frac{T+Z}{2}(1-x)-Z \right\}.%
    \end{split} 
     \]
    The first term in the max is greater than the second term if and only if $ Tx>-Z $. Therefore, the infimum is attained at $ x=-Z/T $ and the 
    extraspherical minimax equals
    \[ 
    A = \frac{Z^2+T^2}{2T}\,,
     \]
     as promised.
\end{proof}

\begin{cor}\label{cor:r-lowerbound}
	If $ |Z|<T $,
	the following recursive relation holds
	\begin{equation}\label{eq:cor-r-lowerbound}
	 r_{k+1}(T,Z) = \min\left\{ \inf_{x\in [-1,1]} \max_{w=\pm 1}  
	\inf_{\substack{0\le t\le T \\
			|Z+tw|< T-t
		}} 
	\left( 
	twx + {r}_k(T-t, Z+tw ) \right),\frac{Z^2+T^2}{2T} \right\}\,. 
	\end{equation}
\end{cor}
Note that if $ |Z+tw|<T-t $, it excludes the possibility of $ T-t=0 $. Because $ T-t>|Z+tw|\geq 0$. %
  We define the normalized regret \[
 u'_k(T,z)\triangleq \frac{1}{T}r_k(T,Tz) \] for $ T>0 $ (In the rest of this subsection, we assume $|z|<1$). 
Plugging this definition into \eqref{eq:cor-r-lowerbound} yields 
 \begin{align*}
 & u'_{k+1}(T,z)\\
  ={}& \min\left\{ \inf_{x\in [-1,1]} \max_{w=\pm 1}  
 \inf_{\substack{0\le t< T \\
 		|Tz+tw|< T-t
 }} 
 \left( 
\frac{twx}{T}  + (1-\frac{t}{T})u'_k(T-t, \frac{Tz+tw}{T-t} ) 
\right),\frac{z^2+1}{2} 
\right\}\,.
  \end{align*}
  For $ 0\le t<T $ and $ |Tz+tw|<T-t $, we define a reparametrization $ 
  z'(t)\triangleq  
  \frac{Tz+tw}{T-t} $. 
The derivative of $ z' $ is $ \frac{dz'}{dt} = \frac{T (w+z)}{(T-t)^2} $. Since 
$ |z|<1 $, $ z'(t) $ is an increasing function if $ w=1 $ and is a decreasing 
function if $ w=-1 $.
If $ w=1 $, the system of inequalities $ 0\le t<T $ and $ |Tz+t|<T-t $ is 
equivalent to $ 0\le t< \frac{T(1-z)}{2} $ and thereby $ z\le z'< 1  $.
If $ w=-1 $, the system of inequalities $ 0\le t<T $ and $ |Tz-t|<T-t $ is 
equivalent to $ 0\le t < \frac{T(1+z)}{2} $ and thereby $ -1<z'\le z $. 
Combining these two cases, we obtain that $ |z'|<1 $ and $ w(z'-z)\ge 0 $.

 The definition of $ z' $ gives $ t= \frac{T (z'-z)}{w+z'} $ and 
 \begin{align*}
 \frac{twx}{T}  + (1-\frac{t}{T})u'_k(T-t, \frac{Tz+tw}{T-t} ) ={}& 
 \frac{(w+z) u'_k\left(\frac{T (w+z)}{w+z'},z'\right)+w x 
 \left(z'-z\right)}{w+z'} \\
 \stackrel{(a)}{=}{}&  \frac{(1+wz) u'_k\left(\frac{T (w+z)}{w+z'},z'\right)+ x 
 \left(z'-z\right)}{1+z'w}\,,
  \end{align*}
  where $ (a) $ is because we multiply the numerator and denominator by $ w $ 
  and use the fact that $ w=\pm 1 $.
  Therefore, we obtain the following corollary.
  \begin{cor}\label{cor:u_prime_z_prime}
  	If $ |z|<1 $, the following recursive relation holds
  	\begin{equation}\label{eq:u_prime_z_prime}
  	  u'_{k+1}(T,z) = \min\left\{ \inf_{x\in [-1,1]} \max_{w=\pm 1}  
  	\inf_{\substack{|z'|<1\\ w(z'-z)\ge 0}} 
  	\frac{(1+wz) u'_k\left(\frac{T (w+z)}{w+z'},z'\right)+ x 
  		\left(z'-z\right)}{1+z'w}
  	,\frac{z^2+1}{2} 
  	\right\}\,. \end{equation}
  \end{cor}
\begin{remark}\label{rmk:t-positive}
	Since both $ z $ and $ z' $ resides in the open interval $ (-1,1) $ and $ w 
	$ is either $ -1 $ or $ 1 $, the quantity $ \frac{w+z}{w+z'} $ is always 
	positive.
\end{remark}

\subsection{Normalized Minimax Regret}

To derive a closed-form lower bound for $ u_k(z) $, we study the boundary 
condition when $ k=1 $.
\begin{lemma}[Boundary condition for $ r_1 $]
	The boundary condition of $ r_k(T,Z) $ when $ k=1 $ is given by \[ 
	r_1(T,Z) = \frac{|Z-T|+|Z+T|}{2}
	\,.
	\]
\end{lemma}
\begin{proof}
	It can be computed directly as below 
	 \[ 
	r_1(T,Z) = \inf_{x\in [-1,1]} \max_{w=\pm 1} (Twx+|Z+Tw|). 
	\]
	We can expand the innermost maximum, which is minimized at 
	$x=\frac{|z-1|-|z+1|}{2}$, as follows
	\[
	\inf_{x\in [-1,1]}\max_{w=\pm 1} (Twx+|Z+Tw|) = \inf_{x\in [-1,1]}\max(Tx+|Z+T|,-Tx+|Z-T|)=\frac{|z-1|+|z+1|}{2}\,.
	\]
	Thus, $r_1(T,Z)= \frac{|Z-T|+|Z+T|}{2}$ as claimed.
\end{proof}
\begin{cor}[Boundary condition for $ u'_1 $]\label{cor:u1}
	If $ |z|<1 $, we have \[u'_1(T,z)= \frac{r_1(T,Tz)}{T} = 
	\frac{|z-1|+|z+1|}{2} = 1\,.
	\] 
\end{cor}
\begin{lemma}\label{lem:independent_T}
	The normalized regret function $ u'_k(T,z) $ does not depend on $ T $. In 
	other words, there exists a function $ u_k(z) $ such that for all $  z\in 
	\bR $ and $ T>0 $, $ u'_k(T,z) =u_k(z)$.
\end{lemma}
\begin{proof}
	If $ |z|\ge 1 $, \cref{lem:absolute-bound} implies that $ 
	u'_k(T,z)=\frac{1}{T}r_k(T,Tz) = \frac{|Tz|}{T}=|z| $,  and thereby we define 
	$ u_k(z)=|z| $ for $ |z|\ge 1 $. 
	
	If $ |z|<1 $,
	we will prove this lemma by induction on $ k $. If we define $ u_1(z)=1 $, 
	\cref{cor:u1} shows that $ u'_1(T,z)=u_1(z) $. Now we assume that $ 
	u'_i(T,z)=u_i(z) $ holds for $ i\ge 1 $. By \cref{cor:u_prime_z_prime}, we 
	have 
	\begin{align*}
	u'_{i+1}(T,z) ={}& \min\left\{ \inf_{x\in [-1,1]} \max_{w=\pm 1}  
	\inf_{\substack{|z'|<1\\ w(z'-z)\ge 0}} 
	\frac{(1+wz) u'_i\left(\frac{T (w+z)}{w+z'},z'\right)+ x 
		\left(z'-z\right)}{1+z'w}
	,\frac{z^2+1}{2} 
	\right\}\\
	={}& \min\left\{ \inf_{x\in [-1,1]} \max_{w=\pm 1}  
	\inf_{\substack{|z'|<1\\ w(z'-z)\ge 0}} 
	\frac{(1+wz) u_i\left(z'\right)+ x 
		\left(z'-z\right)}{1+z'w}
	,\frac{z^2+1}{2} 
	\right\}
	\,. \end{align*}
	Note that the rightmost side does not depend on $ T $. If we define $ 
	u_{i+1}(z) $ by the rightmost side of the above equation, we have $ 
	u'_{i+1}(T,z)=u_{i+1}(z) $. The proof is completed.
\end{proof}

The function that plays a central role in our minimax analysis is the function 
$ u_k(z) $ given in \cref{lem:independent_T}. We call it the
\emph{normalized minimax regret function}. By \cref{lem:absolute-bound}, 
\cref{cor:u_prime_z_prime} and 
\cref{lem:independent_T}, we have an immediate corollary.
\begin{cor}[Recursive relation of normalized minimax 
	regret]\label{cor:recursive-uk}
	If $ |z|\ge 1 $, $ u_k(z)=|z| $.
	If $ |z|<1 $, the normalized minimax regret satisfies
	\begin{equation}\label{eq:recursive-limiting}
	u_{k+1}(z) = \min\left\{ \inf_{x\in [-1,1]} \max_{w=\pm 1}  
	\inf_{\substack{|z'|<1\\ w(z'-z)\ge 0}}
	\frac{(1+wz) u_k(z')+ x 
		(z'-z)}{1+z'w}
	,\frac{z^2+1}{2} 
	\right\}\,.
	\end{equation}
\end{cor}

\begin{lemma}[Boundary condition for $ u_2 $]\label{lem:u2}
	If $ |z|<1 $, we have \[ 
u_2(z)=	\inf_{x\in [-1,1]} \max_{w=\pm 1}  
	\inf_{\substack{|z'|<1\\ w(z'-z)\ge 0}} 
	\frac{(1+wz) + x 
		(z'-z)}{1+z'w} =\frac{z^2+1}{2}\,.
	 \]
\end{lemma}
\begin{proof}
	Plugging $ k=1 $ and $ u_1(z)=1 $ for $ |z|<1 $ into 
	\eqref{eq:recursive-limiting} gives \[ 
		u_{2}(z) = \min\left\{ \inf_{x\in [-1,1]} \max_{w=\pm 1}  
	\inf_{\substack{|z'|<1\\ w(z'-z)\ge 0}} 
	\frac{(1+wz) + x 
		(z'-z)}{1+z'w}
	,\frac{z^2+1}{2} 
	\right\}\,.
	 \]
	 We define the function $ f(z')= \frac{(1+wz) + x 
	 	(z'-z)}{1+z'w} $. Differentiating this function yields \[ 
 	\frac{df}{dz'}=-\frac{(w-x) (w z+1)}{\left(w z'+1\right)^2}\,.
 	 \]
 	 Since $ w=\pm 1 $, $ |x|\le 1 $ and $ |z|<1 $, we have $ wz+1>0 $ and that 
 	 the sign of $ \frac{df}{dz} $ is the same as $ -\sign(w) $, or 0 if $w=x$. Therefore, the 
 	 function is non-decreasing if $ w=-1 $ and is non-increasing if $ w=1 $. 
 	 The innermost infimum is attained as $ z'\to w $. As a result, we deduce 
 	 \[ \inf_{\substack{|z'|<1\\ w(z'-z)\ge 0}}
 	 \frac{(1+wz) + x 
 	 	(z'-z)}{1+z'w} = \frac{1}{2} (w (x+z)-x z+1)\,, \]
  	and \begin{align*}
  	\max_{w=\pm 1}  
  	\inf_{\substack{|z'|<1\\ w(z'-z)\ge 0}} 
  	\frac{(1+wz) + x 
  		(z'-z)}{1+z'w}={}&\frac{1}{2}\max_{w=\pm 1} (w (x+z)-x z+1)\\
  	 ={}& 
  		\frac{1}{2}\max\{ x (1-z)+z+1, -x (z+1)-z+1\}\,.
  	 \end{align*}
  	 The outermost infimum is attained when $ x (1-z)+z+1= -x (z+1)-z+1 $, or 
  	 equivalently, at $ x=-z $. Therefore, we have
  	 \[ 
  	 \inf_{x\in [-1,1]} \max_{w=\pm 1}  
  	 \inf_{\substack{|z'|<1\\ w(z'-z)>0}}
  	 \frac{(1+wz) + x 
  	 	(z'-z)}{1+z'w} = \frac{1}{2}[x(1-z)+z+1]_{x=-z}=\frac{z^2+1}{2}\,.
  	  \]
  	  Therefore $ u_2(z) $ also equals $ \frac{z^2+1}{2} $.
\end{proof}
\begin{lemma}[Monotonicity in $ k $]\label{lem:monotone-in-k}
	The sequence of functions $ u_{k}(z) $ is non-increasing pointwise on $ 
	(-1,1) $, \ie, $ 
	u_{k+1}(z)\le u_k(z) $ for $ |z|<1 $.
\end{lemma}
\begin{proof}
	By the definition of $ r_k $ in \eqref{eq:def_rk}, we see that a player's 
	strategy with $ k $ switches can be viewed as a strategy with $ k+1 $ 
	switches. Therefore, we have $ r_{k+1}(T,z)\le r_k(T,z) $ and therefore $ 
u_{k+1}(z)	= \frac{1}{T}r_{k+1}(T,z)\le \frac{1}{T}r_k(T,z) = u_k(z)$. 
\end{proof}
Combining \cref{lem:u2} and \cref{lem:monotone-in-k} implies the following 
corollary immediately.
\begin{cor}\label{cor:u2-upperbound}
	For all $ k\ge 2 $ and $ |z|<1 $, $ u_k(z)\le \frac{z^2+1}{2} $.
\end{cor}
\cref{lem:improved-recursive-uk} improves the recursive relation in 
\cref{cor:recursive-uk} by removing the operation of taking the minimum with $ 
\frac{z^2+1}{2} $. In fact, \cref{lem:improved-recursive-uk} and 
\cref{cor:recursive-uk} are mathematically equivalent since we will show that 
the first term in the minimum operator in \eqref{eq:recursive-limiting} is 
always less than or equal to the second term $ \frac{z^2+1}{2} $.
\begin{lemma}[Improved recursive relation of $ u_k 
$]\label{lem:improved-recursive-uk}
	For all $ k\ge 1 $ and $ |z|<1 $, $ u_k(z) $ obeys the recursive relation 
	\begin{equation}\label{eq:improved-recursive-limiting}
	u_{k+1}(z) = \inf_{x\in [-1,1]} \max_{w=\pm 1}  
	\inf_{\substack{|z'|<1\\ w(z'-z)\ge 0}} 
	\frac{(1+wz) u_k(z')+ x 
		(z'-z)}{1+z'w} 
	\,.
	\end{equation}
\end{lemma}
\begin{proof}
	If $ k=1 $, the desired equation holds due to \cref{lem:u2}. If $ k\ge 2 
	$, \cref{cor:u2-upperbound} shows $ u_k(z)\le \frac{z^2+1}{2} $. If we take 
	$ z'=z $, we have \begin{align*}
	\inf_{\substack{|z'|<1\\ w(z'-z)\ge 0}} 
	\frac{(1+wz) u_k(z')+ x 
		(z'-z)}{1+z'w} \le{}& \inf_{\substack{|z'|<1\\ w(z'-z)\ge 0}}
	\frac{(1+wz)\frac{z^2+1}{2} + x 
	(z'-z)}{1+z'w}\\
 \le{}& \left[\frac{(1+wz)\frac{z^2+1}{2}+ x 
		(z'-z)}{1+z'w}\right]_{z'=z}\\
	={}& \frac{z^2+1}{2}\,.
	 \end{align*}
	By \cref{cor:recursive-uk}, we deduce \begin{align*}
	u_{k+1}(z) ={}& \min\left\{ \inf_{x\in [-1,1]} \max_{w=\pm 1}  
	\inf_{\substack{|z'|<1\\ w(z'-z)\ge 0}} 
	\frac{(1+wz) u_k(z')+ x 
		(z'-z)}{1+z'w}
	,\frac{z^2+1}{2} 
	\right\}\\
	={}&\inf_{x\in [-1,1]} \max_{w=\pm 1}  
	\inf_{\substack{|z'|<1\\ w(z'-z)\ge 0}} 
	\frac{(1+wz) u_k(z')+ x 
		(z'-z)}{1+z'w} \,.
	\end{align*}
\end{proof}

\subsection{Fugal Operator and Quadratic Lower Bound}
The recursive relation in \cref{lem:improved-recursive-uk} relates two consecutive $u_k$'s. In light of this recursive relation, we define the fugal operator that sends $u_k$ to $u_{k+1}$.
\begin{defn}[Fugal operator]
    Let $ C[-1,1] $ denote the space of continuous functions on $ [-1,1] $. The \emph{fugal operator} $ \cT:C[-1,1]\to C[-1,1] $ is defined by \[ 
(\cT f)(z) \triangleq \inf_{x\in [-1,1]} \max_{w=\pm 1}  
\inf_{\substack{|z'|<1\\ w(z'-z)\ge 0}} 
\frac{(1+wz) f(z')+ x 
	(z'-z)}{1+z'w} 
\,,
 \]
 where $ f\in C[-1,1] $.
\end{defn}
\begin{remark}
 Using this notation, \cref{lem:improved-recursive-uk} 
 can be re-written in a more compact way \[ 
 u_{k+1} = \cT u_k\,.
  \]
\end{remark}
\begin{remark}[Monotonicity of fugal operator]
 If $f,g\in C[-1,1]$ satisfy $f(z)\le g(z)$ for all $z\in [-1,1]$, we have the 
 following 
 inequality $\frac{(1+wz) f(z')+ x 
	(z'-z)}{1+z'w} \ge \frac{(1+wz) g(z')+ x 
	(z'-z)}{1+z'w} $. This is because $1+wz\ge 0$ holds for any $w=\pm 1$ 
	and $|z|\le 1$, and $ 1+z'w>0 $ holds for any $ w=\pm 1 $ and $ |z'|<1 $. 
	As 
	a 
	result, we have $(\cT f)(z)\ge (\cT g)(z)$ for all $z\in 
	[-1,1]$.
\end{remark}
 
Before deriving a lower bound for $ u_{k} $, we study the action of the fugal 
operator on \emph{quadratic lower bound functions}. 
\begin{defn}[Quadratic lower bound functions]
    The quadratic lower bound 
functions $ a_k(z) $ on $ [-1,1] $ are defined by 
 by 
 $ a_1(z)=1 $ and for $ i\ge 2 $ \[ 
 a_i(z)=\begin{cases}
 \frac{\sqrt{i/2}z^2+\sqrt{2/i}}{2},&  |z|<\sqrt{2/i} \,,\\
 |z|,& |z|\ge \sqrt{2/i}\,.
 \end{cases}
 \] 
\end{defn}
\begin{remark}[Continuity]
    If $z=\pm \sqrt{2/i}$, the expression $\frac{\sqrt{i/2}z^2+\sqrt{2/i}}{2} = \sqrt{2/i}= |z|$. The quadratic lower bound function $a_i$ is continuous on $[-1,1]$.
\end{remark}
\begin{remark}
	If $ i=1,2 $, the quadratic lower bound functions agree with the normalized 
	minimax regret functions, \ie, $ a_1(z)=u_1(z) =1$ and $ 
	a_2(z)=u_2(z)=\frac{z^2+1}{2} $.
\end{remark}

 We will show later in \cref{lem:alt} that the quadratic lower bound functions provide indeed a lower bound for $u_k$'s, \ie, $a_i(z)\le u_i(z)$. This result will be proved in two steps. The first step is to obtain the closed-form expression of $\cT a_i$ (\cref{prop:fugal-operator-quadratic}) and the second step is to show that the fugal operator interlaces $a_i$, in other words, $\cT a_i\ge a_{i+1}$ (\cref{lem:fugal-monotone}). Then we can argue that $u_{i+1}= \cT u_i \ge \cT a_i \ge a_{i+1}$, provided that $u_i\ge a_i$, where the first inequality is due to the monotonicity of the fugal operator and the second inequality is because the fugal operator interlaces $a_i$. Therefore $a_i\le u_i$ for all $i$ can be obtained by induction. 
 
\begin{prop}[Fugal operator on quadratic lower bound 
functions]\label{prop:fugal-operator-quadratic}
	If $ i \ge 2 $,  it holds that \[ 
	(\cT a_i)(z) = \begin{cases}
	\sqrt{\frac{i}{2}}\left[z^2-1+\sqrt{1+\frac{2}{i}-z^2}\right],& |z|\le 
	\sqrt{2/i}\,,\\
	|z|,& |z|> \sqrt{2/i}\,.\\
	\end{cases}
	\]
\end{prop}
Before presenting the proof of \cref{prop:fugal-operator-quadratic}, we need several lemmas.
\begin{lemma}\label{lem:simplified-t-ai}
	If $ i\ge 2 $ and we define \begin{align*}
		z_+\triangleq{}&
		\sqrt{1+\frac{2}{i}-2\sqrt{\frac{2}{i}}x}-1 \\
		 z_-\triangleq{}&
		1-\sqrt{1+\frac{2}{i}+2\sqrt{\frac{2}{i}}x}\\
		 g_+(x,z) \triangleq{}&
		x+(1+z)\left[\frac{a_i(z')-x}{1+z'}\right]_{z'=\max\{z,z_+\}}\\
		g_-(x,z)\triangleq{}&
		-x+(1-z)\left[\frac{a_i(z')+x}{1-z'}\right]_{z'=\min\{z,z_-\}}  \,,
	\end{align*}   
	 the following equation holds 
	\begin{equation}\label{eq:lz}
	(\cT a_i)(z)  = \inf_{x\in [-1,1]} \max\left\{ 
	g_+(x,z)
	,g_-(x,z) \right\} 
	\,.
	\end{equation}
\end{lemma}
\begin{proof}
	Recalling the definition of the fugal operator gives
	\begin{align*}
	& (\cT a_i)(z)\\
	={}& \inf_{x\in [-1,1]} \max_{w=\pm 1}  
	\inf_{\substack{|z'|<1\\ w(z'-z)\ge 0}} 
	\frac{(1+wz) a_i(z')+ x 
		(z'-z)}{1+z'w}\\
	={}& \inf_{x\in [-1,1]} \max\left\{
	\inf_{z':z\le z'<1} 
	\frac{(1+z) a_i(z')+ x 
		(z'-z)}{1+z'}, \inf_{-1<z'\le z} 
	\frac{(1-z) a_i(z')+ x 
		(z'-z)}{1-z'}\right\}
	\end{align*}
	We observe that if $ |z'|<1 $, the following equations hold \begin{align*}
	\frac{(1+z) a_i(z')+ x 
		(z'-z)}{1+z'}={}& \frac{1+z}{1+z'}(a_i(z')-x)+x\,, \\
	\frac{(1-z) a_i(z')+ x 
		(z'-z)}{1-z'}={}& \frac{1-z}{1-z'}(a_i(z')+x)-x\,.
	\end{align*}
	Therefore, we simplify the innermost infima \begin{align*}
	\inf_{z': z\le z'< 1}  
	\frac{(1+z) a_i(z')+ x 
		(z'-z)}{1+z'} ={}& x+ (1+z)\inf_{z':z\le z'<1}  
	\frac{a_i(z')-x}{1+z'}\,,\\
	\inf_{z':-1<z'\le z} 
	\frac{(1-z) a_i(z')+ x 
		(z'-z)}{1-z'} ={}& -x+(1-z)\inf_{z':-1< z'\le z}
	\frac{a_i(z')+x}{1-z'}\,.
	\end{align*}
	We define two functions $ f_+(z')=\frac{a_i(z')-x}{1+z'} $ and $ 
	f_-(z')=\frac{a_i(z')+x}{1-z'} $. 
	
	First, we assume $ |z'|<\sqrt{2/i} $.
	Differentiating $ f_+ $ gives \[ 
	\frac{df_+}{dz'}=\frac{\sqrt{2i}  z' \left(z'+2\right)-2 
		\sqrt{\frac{2}{i}}+4 x}{4 \left(z'+1\right)^2}\,.
	\]
	Setting $  \frac{df_+}{dz'}>0 $ yields  \[ 
	(z'+1)^2>1+\frac{2}{i}-2\sqrt{\frac{2}{i}}x\,.
	\]
	The fact that $ |x|\le 1 $ implies $ 1+\frac{2}{i}-2\sqrt{\frac{2}{i}}x\ge 
	1+\frac{2}{i}-2\sqrt{\frac{2}{i}}=\left( \sqrt{\frac{2}{i}}-1 \right)^2\ge 
	0 
	$. Therefore $ \frac{df_+}{dz'}>0 $ is equivalent to \[ 
	z'+1=|z'+1|>\sqrt{1+\frac{2}{i}-2\sqrt{\frac{2}{i}}x}\,.
	\]
	In other words, $ \frac{df_+}{dz'}>0 $ if and only if $ 
	z'>z_+\triangleq \sqrt{1+\frac{2}{i}-2\sqrt{\frac{2}{i}}x}-1 $. 
	Our assumption $ i\ge 2 $ implies \[ 
	\sqrt{1+\frac{2}{i}-2\sqrt{\frac{2}{i}}x}\ge 
	\sqrt{1+\frac{2}{i}-2\sqrt{\frac{2}{i}}} = \left|1- 
	\sqrt{\frac{2}{i}}\right| = 1- 
	\sqrt{\frac{2}{i}}\,.
	\]
	As a result, $ z_+\ge -\sqrt{\frac{2}{i}} $.
	Since \[ 
	\sqrt{1+\frac{2}{i}-2\sqrt{\frac{2}{i}}x}\le 
	\sqrt{1+\frac{2}{i}+2\sqrt{\frac{2}{i}}} =\left| \sqrt{\frac{2}{i}}+1\right|
	=1+\sqrt{\frac{2}{i}}\,,
	\]
	we obtain the upper bound $ z_+\le \sqrt{2/i} $, 
	where the second inequality uses the assumption $ i\ge 2 $. Thus 
	we are certain that $ |z_+|\le \sqrt{2/i} $.
	
	If $ z'\ge \sqrt{2/i} $, the function $ f_+ $ becomes $ 
	f_+(z')=\frac{z'-x}{1+z'}= 1-\frac{1+x}{1+z'}
	$, which is non-decreasing in $ z' $. 
	On the other hand, if $ z'\le -\sqrt{2/i} $, the function $ 
	f_+ $ becomes $ f_+(z')=\frac{-z'-x}{1+z'}=-1+\frac{1-x}{1+z'} $, which is 
	non-increasing in $ z' $. It follows that 
	$ f_+ $ is non-increasing on $ (-1,z_+) $ and non-decreasing on $ (z_+,1) 
	$. Therefore we can solve the infimum \[ 
	\inf_{z':z\le z'<1}
	\frac{a_i(z')-x}{1+z'} = 
	\left[\frac{a_i(z')-x}{1+z'}\right]_{z'=\max\{z,z_+\}}\,.
	\]
	
	If $ |z'|<\sqrt{2/i} $, the derivative of $ f_- $ with respect to $ z' $ 
	is \[ 
	\frac{df_-}{dz'} = \frac{4 \sqrt{i} x-\sqrt{2} i \left(z'-2\right) 
		z'+2 \sqrt{2}}{4 \sqrt{i} \left(z'-1\right)^2}\,.
	\]
	Setting the derivative greater than $ 0 $ yields \[ 
	(z'-1)^2 < 1+\frac{2}{i}+2\sqrt{\frac{2}{i}}x\,.
	\]
	The right-hand side is at least $ 
	1+\frac{2}{i}-2\sqrt{\frac{2}{i}}=(1-\sqrt{2/i})^2\ge 0 $. Since the 
	right-hand side is non-negative and $ z<1 $, we have \[ 
	z'>z_-\triangleq 1-\sqrt{1+\frac{2}{i}+2\sqrt{\frac{2}{i}}x}\,.
	\]
	If $ z'\ge \sqrt{2/i} $, the function $ f_- $ equals $ 
	\frac{z'+x}{1-z'}=-1+\frac{x+1}{1-z'} $, which is non-decreasing in $ z' $. 
	On 
	the other hand, if $ z'\le -\sqrt{2/i} $, the function $ f_- $ equals $ 
	\frac{-z'+x}{1-z'}=1+\frac{x-1}{1-z'} $, which is non-increasing in $ z' $. 
	It 
	follows that $ f_- $ is non-increasing on $ (-1,z_-) $ and non-decreasing 
	on $ 
	(z_-,1) $. Thus we solve the other infimum \[ 
	\inf_{z':-1< z'\le z}  
	\frac{a_i(z')+x}{1-z'} = 
	\left[\frac{a_i(z')+x}{1-z'}\right]_{z'=\min\{z,z_-\}}\,.
	\]
	The equation \eqref{eq:lz} is thereby obtained by combining our results 
	regarding the two infima.
\end{proof}
\begin{lemma}\label{lem:zp-ge-zm}
	If $ z_+ $ and $ z_- $ are as defined in \cref{lem:simplified-t-ai}, we 
	have $ z_+\ge z_- $. 
\end{lemma}
\begin{proof}
We compute the difference of $ z_+ $ and $ z_- $
	 \[ 
	z_+-z_-= \sqrt{1+\frac{2}{i}+2\sqrt{\frac{2}{i}}x} + 
	\sqrt{1+\frac{2}{i}-2\sqrt{\frac{2}{i}}x}-2\,.
	\]
	To show that $ z_+-z_-\ge 0 $, it is sufficient to show that \[ 
	\left(\sqrt{1+\frac{2}{i}+2\sqrt{\frac{2}{i}}x} + 
	\sqrt{1+\frac{2}{i}-2\sqrt{\frac{2}{i}}x}\right)^2 \ge 4\,.
	\]
	The left-hand side equals \begin{align*}
	2\left(1+\frac{2}{i}\right)+2\sqrt{ \left( 1+\frac{2}{i} \right)^2-4\cdot 
		\frac{2}{i}x^2 } \ge{}& 2\left(1+\frac{2}{i}\right)+2\sqrt{ \left( 
		1+\frac{2}{i} \right)^2-4\cdot 
		\frac{2}{i} }\\
		={}& 2\left( 1+\frac{2}{i}+\left|1-\frac{2}{i}\right| 
	\right)=4\,,
	\end{align*}
	where the last inequality is because $ i\ge 2 $ and thus $ 
	1-\frac{2}{i}\ge 0 $. Therefore we establish $ z_+\ge z_- $.
\end{proof}
\begin{lemma}\label{lem:h-unique-zero}
	Given $ |z|\le 1 $, the function $ h_z(x)\triangleq g_+(x,z)-g_-(x,z) $ has 
	a unique zero $ x=x_0(z) $ on $ [-1,1] $ and it satisfies \[ 
	(\cT a_i)(z) = \inf_{x\in [-1,1]} \max\{g_+(x,z),g_-(x,z) \}= 
	g_+(x_0(z),z)=g_-(x_0(z),z) 
	\,.
	\]
\end{lemma}
\begin{proof}
		Recall that $ z_+ $ and $ z_- $ are functions of $ x $ but do not 
		rely 
	on 
	$ 
	z $. They obey an additional relation \[ z_+(x)+z_-(-x)=0\,. \]
	Their inverses are \[ z_+^{-1}(z)=\frac{ 2-i z^2-2iz}{2 
		\sqrt{2i}},\quad z_-^{-1}(z)= \frac{ iz^2-2iz-2}{2 
		\sqrt{2i}}\,, \]
	respectively. Both inverse functions are strictly decreasing. 
	Using the relation $ z_+(x)+z_-(-x)=0 $, since $ \max\{-z,z_+(-x)\}= 
	\max\{-z,-z_-(x)\} = 
	-\min\{z,z_-(x)\} $ and $ a_i $ is an even function, we have 
	\begin{equation}\label{eq:gpm}
	\begin{split} 
	g_+(-x,-z) ={}& 
	-x+(1-z)\left[\frac{a_i(z')+x}{1+z'}\right]_{z'=\max\{-z,z_+(-x)\}}\\
	={}& -x+(1-z)\left[\frac{a_i(z')+x}{1+z'}\right]_{z'=-\min\{z,z_-(x)\}}\\
	={}& -x+(1-z)\left[\frac{a_i(-z')+x}{1-z'}\right]_{z'=\min\{z,z_-(x)\}}\\
	={}& g_-(x,z)
	\,.
	\end{split}
	\end{equation}
	
	If $ z\ge z_+ $, we have $ z'=z $ and $ g_+(x,z)=a_i(z) $. In this case, $ 
	g_+ 
	$ is a constant function with respect to $ x $. If $ z<z_+ $, we have $ 
	z'=z_+ 
	$ and $ g_+(x,z)= x+(1+z)\frac{a_i(z_+)-x}{1+z_+} $. Since $ |z_+|\le 
	\sqrt{2/i} 
	$, we have $ a_i(z_+)=\frac{\sqrt{i/2}z_+^2+\sqrt{2/i}}{2} $ and \[ 
	g_+(x,z) = x+(1+z)\frac{(\sqrt{i/2}z_+^2+\sqrt{2/i})/2-x}{1+z_+}\,.
	\]
	Differentiating $ g_+ $ yields \[ 
	\frac{\partial g_+}{\partial x}=1-\frac{i (z+1)}{\sqrt{i \left(-2 \sqrt{2i} 
			x+i+2\right)}}\,.
	\]
	Since $ z<z_+ =\sqrt{1+\frac{2}{i}-2\sqrt{\frac{2}{i}}x}-1$, we have \[ 
	z+1<\sqrt{1+\frac{2}{i}-2\sqrt{\frac{2}{i}}x}\,,
	\]
	which, in turn, implies \[ 
	\frac{\partial g_+}{\partial x} > 1-\frac{i 
		\sqrt{1+\frac{2}{i}-2\sqrt{\frac{2}{i}}x}}{\sqrt{i \left(-2 \sqrt{2i} 
			x+i+2\right)}}=0\,.
	\]
	Therefore, $ g_+(x,z) $ is strictly increasing in $ x $ if $ z<z_+(x) $ 
	(\ie, $ x<z_+^{-1}(z) $) and is 
	constant with respect to $ x $ if $ 
	z\ge z_+(x) $ (\ie, $ x\ge z_+^{-1}(z) $). Furthermore, we verify that $ 
	g_+(-1,z)=z $ and 
	$ 
	g_+(1,z)=a_i(z) $.
	
	In light of the relation \eqref{eq:gpm}, we derive the property of $ g_- $. 
	The 
	function $ g_-(x,z) $ is strictly decreasing if $ -z<z_+(-x) $, or 
	equivalently, $ z>z_-(x) $ (\ie, $ x>z_-^{-1}(z) $). It stays at $ a_i(z) $ 
	if 
	$ z\le z_-(x) $ (\ie, $ x\le z_-^{-1}(z) $). 
	Furthermore, we have $ g_-(-1,z)=g_+(1,-z) = a_i(z) $ and $ 
	g_-(1,z)=g_+(-1,-z)=-z $.
	
	Let $ h_z(x)\triangleq g_+(x,z)-g_-(x,z) $ be the difference of these two 
	functions. Since $ g_+ $ is non-decreasing in $ x $ and $ g_- $ is 
	non-increasing in $ x $, the function $ h_z $ is non-decreasing in $ x $. 
	Then we check 
	the 
	value of $ h_z $ at $ x=-1 $ and $ x=1 $. We have $ 
	h_z(-1)=g_+(-1,z)-g_-(-1,z)=z-a_i(z) $ and $ 
	h_z(1)=g_+(1,z)-g_-(1,z)=a_i(z)+z 
	$. Their product is $ h_z(-1)h_z(1)=z^2-a_i^2(z) $, which is non-positive 
	because $ |z|\le a_i(z) $. The continuity of $ h_z $ implies the existence 
	of a 
	zero on $ [-1,1] $. Next, we will show the uniqueness of the zero. Since $ 
	g_+ 
	$ is strictly increasing with respect to $ x $ at the initial stage when $ 
	x<z_+^{-1}(z) $ 
	and stays constant when
	$ x\ge z_+^{-1}(z) $, 
	and $ g_- $ is constant with respect to $ x $ at the initial stage when
	$ x\le z_-^{-1}(z) $
	and strictly decreases when $ x>z_-^{-1}(z) $,
	the only possibility of having more than one zero is that $ 
	z_-^{-1}(z)>z_+^{-1}(z) $ and that the set $ 
	R=[z_+^{-1}(z),z_-^{-1}(z)]\cap 
	[-1,1] $ contains more than one point. The inequality $ 
	z_-^{-1}(z)>z_+^{-1}(z) $ is equivalent to $ |z|>\sqrt{2/i} $. A necessary 
	condition for the set $ R $ containing more than one point is that both $ 
	z_-^{-1}(z)>-1 $ and $ z_+^{-1}(z)<1 $ holds. If $ i=2 $, $ 
	|z|>\sqrt{2/i}=1 $ 
	will never happen. If $ i>2 $, the expression $ 
	z_-^{-1}(z)>-1 $ is equivalent to $ -1\leq z< \sqrt{\frac{2}{i}} $ while 
	the 
	expression $ z_+^{-1}(z)<1 $ is equivalent to $ - 
	\sqrt{\frac{2}{i}}<z\leq 1 $. However, the three inequalities  $ 
	|z|>\sqrt{2/i} 
	$, $ 
	-1\leq z< \sqrt{\frac{2}{i}} $, and $ - 
	\sqrt{\frac{2}{i}}<z\leq 1 $ cannot be satisfied simultaneously.
	Therefore, we show that $ h_z(x) $ has a unique zero on $ [-1,1] $.
	Let $ x_0(z) $ denote the unique zero, which is a function of $ z $. By its 
	definition, the two functions $ g_+ $ 
	and $ g_- $ are equal at $ x=x_0 $. Since $ h_z $ is non-decreasing with 
	respect to $ x $ and $ x_0 $ is the unique zero, we know that $ 
	g_+(x)>g_-(x) $ 
	if $ x>x_0 $ and $ g_+(x)<g_(x) $ if $ x<x_0 $. Therefore, by 
	\cref{lem:simplified-t-ai}, $ (\cT a_i)(z) $ 
	equals \[ 
	(\cT a_i)(z) = \inf_{x\in [-1,1]} \max\{g_+(x,z),g_-(x,z) \}= 
	g_+(x_0(z),z)=g_-(x_0(z),z) 
	\,.
	\]
\end{proof}

We are now ready to prove \cref{prop:fugal-operator-quadratic}.
\begin{proof}[Proof of \cref{prop:fugal-operator-quadratic}]
	In light of \cref{lem:h-unique-zero}, we compute the closed-form expression 
	of $ \cT a_i $ by verifying
	 that \[ 
	x_0(z)=\begin{cases}
	-\frac{z \sqrt{-i z^2+i+2}}{\sqrt{2}},& |z|\le \sqrt{2/i}\,,\\
	-\sign(z),& |z|> \sqrt{2/i}\,,
	\end{cases}
	\] 
	is 
	the unique zero of $ h_z(x) $. We consider two cases $ |z|\le \sqrt{2/i} $ 
	and $ |z|>\sqrt{2/i} 
	$. 
	
	\paragraph{Case 1: $ |z|> \sqrt{2/i} $.}
	Let us begin with the case where $ |z|>\sqrt{2/i} $. In this case, $ 
	x_0(z)=-\sign(z) $ and it is indeed on $ [-1,1] $. We further divide this 
	case 
	into two sub-cases where $ z>\sqrt{2/i} $ and $ z<-\sqrt{2/i} $, 
	respectively. 
	
	\subparagraph{Case 1.1: $ z>\sqrt{2/i} $.}
	If $ z>\sqrt{2/i} $, and since $ |z_+|\le \sqrt{2/i} $, we know that $ 
	z>z_+ $ 
	and $ \max\{ z,z_+ \}=z $. In this sub-case, we have $ x_0=-1 $ and  \[ 
	g_+(x_0,z)= -1+(1+z)\frac{a_i(z)+1}{1+z} = a_i(z)=z\,.
	\]
	Since $ |z_-|\le \sqrt{2/i} $ and $ z>z_- $, we have $ \min\{z,z_-\}=z_- $. 
	Therefore, we can compute \[ z_-=z_-(-1) = 
	1-\sqrt{1+\frac{2}{i}-2\sqrt{\frac{2}{i}}}=\sqrt{\frac{2}{i}}  \] and \[ 
	g_-(x_0,z)=1+(1-z)
	\frac{a_i(z_-)-1}{1-z_-}
	= 1+(1-z)
	\frac{\sqrt{2/i}-1}{1-\sqrt{2/i}} = z \,.
	\]
	
	\subparagraph{Case 1.2: $ z<-\sqrt{2/i} $.}
	In the second sub-case, we assume that $ z<-\sqrt{2/i} $. In this sub-case, 
	we have $ x_0=1 $, $ 
	\max\{z,z_+\} =z_+= -\sqrt{2/i} $, and $ \min\{z,z_-\}=z $. The function $ 
	g_+(x_0,z) $ equals \[ 
	g_+(x_0,z)= 1+(1+z)\frac{a_i(-\sqrt{2/i})-1}{1-\sqrt{2/i}} = -z\,,
	\]
	where the function $ g_-(x_0,z) $ equals \[ 
	g_-(x_0,z) = -1+(1-z)\frac{a_i(z)+1}{1-z} = a_i(z)=-z\,.
	\]
	Therefore, $ x_0=-\sign(z) $ is indeed the root when $ |z|>\sqrt{2/i} $. 
	Combining these two sub-cases, we deduce that if $ |z|> \sqrt{2/i} $, 
	\begin{equation}\label{eq:L_large_abs} 
	(\cT a_i)(z)= |z|\,.
	\end{equation}
	
	\paragraph{Case 2: $ |z|\le \sqrt{2/i} $.}
	The case that needs more work is $ |z|\le \sqrt{2/i} $. In this case, the 
	root function is $ 
	x_0(z)=-\frac{z \sqrt{-i z^2+i+2}}{\sqrt{2}} $. First, let us check that $ 
	x_0(z) $ resides on $ [-1,1] $. Since $ |z|\le \sqrt{2/i} $, it holds that 
	$ 
	(z^2-1)(iz^2-2)\ge 0 $. Expanding it and re-arranging the terms yields $ 
	z^2(i+2-iz^2)\le 2 $ and therefore $ |x_0(z)|\le 1 $. 
	
	We claim $ z_-^{-1}(z)\le x_0(z)\le z_+^{-1}(z) $.
	Notice the following factorization \[ 
	x_0(z)-z_-^{-1}(z)=\frac{\left(\sqrt{-i z^2+i+2}-\sqrt{i}\right) 
		\left(\sqrt{-i z^2+i+2}-2 \sqrt{i} z+\sqrt{i}\right)}{2\sqrt{2i}}\,.
	\]
	The first term $ \sqrt{-i z^2+i+2}-\sqrt{i} $ is a decreasing function with 
	respect to $ z^2 $. Since $ z^2 \le 2/i $, the first term is non-negative. 
	Let $ s(z)\triangleq \sqrt{-i z^2+i+2}-2 \sqrt{i} z+\sqrt{i} $ denote the 
	second term. Its derivative is $ s'(z) = -\frac{i z}{\sqrt{-i z^2+i+2}}-2 
	\sqrt{i} $. If $ z\ge 0 $, we see that $ s'(z) < 0 $. Since $ i\ge 2 $, it 
	holds that $ 2/i\le 4(1+2/i)/5 $. In light of the assumption $ z^2 \le 2/i 
	$, we get $ z^2 \le 4(1+2/i)/5 $. Re-arranging the terms gives $ z^2\le 
	4(1+2/i-z^2) $. If $ z<0 $, taking the square root of both sides yields $ 
	-z\le 2\sqrt{1+2/i-z^2} $. Re-arranging the terms again proves that if $ 
	z<0 
	$, $ s'(z)\le 0 $. Since $ s(z) $ is a continuous function, we show that $ 
	s 
	$ is a non-increasing function on $ [-\sqrt{2/i},\sqrt{2/i}] $ and that for 
	any $ z\in [-\sqrt{2/i},\sqrt{2/i}] $, we have $ s(z)\ge s(\sqrt{2/i}) = 
	2(\sqrt{i}-\sqrt{2})\ge 0 $. Thus we show that $ x_0(z)\ge z_-^{-1}(z) $.
	
	Next we need to show that $ x_0(z)\le z_+^{-1}(z) $. Notice the following 
	factorization \[ 
	z_+^{-1}(z)-x_0(z) = \frac{\left(\sqrt{-i z^2+i+2}-\sqrt{i}\right) 
		\left(\sqrt{-i z^2+i+2}+2 \sqrt{i} z+\sqrt{i}\right)}{2  \sqrt{2i}}\,.
	\]
	We observe that $ z_+^{-1}(z)-x_0(z) = x_0(-z)-z_-^{-1}(-z)\ge 0 $ since $ 
	-z $ is also on $ [-\sqrt{2/i},\sqrt{2/i}] $. Therefore we conclude that 
	for all $ z\in [-\sqrt{2/i},\sqrt{2/i}] $, the inequality $ z_-^{-1}(z)\le 
	x_0(z)\le z_+^{-1}(z) $ holds. This inequality implies \[ z_-(x_0(z)) \le z 
	\le z_+(x_0(z))\,. \]
	
	In what follows, we compute the exact values of $ z_+(x_0(z)) $ and $ 
	z_-(x_0(z)) $.   We first compute $z_+(x_0(z))$ \begin{align*}
	z_+(x_0(z)) ={}& \sqrt{2  z \sqrt{- z^2+1+\frac{2}{i}}+\frac{2}{i}+1}-1\\
	={}& \left| \sqrt{- z^2+1+\frac{2}{i}}+z \right|-1\\
	={}&  \sqrt{- 
		z^2+1+\frac{2}{i}}+z -1\,.
	\end{align*}
	The last equality is because $ \sqrt{- 
		z^2+1+\frac{2}{i}}+z\ge 0 $. To see this, we define $ s_1(z)\triangleq 
	\sqrt{- 
		z^2+1+\frac{2}{i}}+z $. Its second derivative is $ 
	s''_1(z)=\frac{i+2}{\sqrt{\frac{2}{i}-z^2+1} \left(i 
		\left(z^2-1\right)-2\right)}\le 0 $. Therefore, for any $ z\in 
	[-\sqrt{2/i},\sqrt{2/i}] $, its derivative satisfies $ s'_1(z)\ge 
	s'_1(\sqrt{2/i}) = 1- \sqrt{2/i} \ge 0$. As a result, for any $ z\in 
	[-\sqrt{2/i},\sqrt{2/i}] $, $ s_1(z)\ge 
	s_1(-\sqrt{2/i})=1-\sqrt{2/i}\ge 0 
	$.
	On the other hand, we compute $ z_-(x_0(z)) $ \begin{align*}
	z_-(x_0(z)) ={}& 1-\sqrt{-2  z \sqrt{- 
			z^2+1+\frac{2}{i}}+\frac{2}{i}+1}\\
			={}& 1- \left| \sqrt{- 
		z^2+1+\frac{2}{i}} -z\right| \\
		={}& -\sqrt{\frac{2}{i}-z^2+1}+z+1\,.
	\end{align*}
	The last inequality is because $ \sqrt{- 
		z^2+1+\frac{2}{i}} -z=s_1(-z)\ge 0 $. 
	
	Now, let us compute $ g_+(x_0(z),z) $ and $ g_-(x_0(z),z) $.
	  Since $ 
	\max\{z,z_+(x_0)\}= z_+(x_0) $ and $ a_i(z_+) = 
	\frac{\sqrt{i/2}z_+^2+\sqrt{2/i}}{2} $ (this is because $ |z_+|\le 
	\sqrt{2/i} $), plugging $ 
	z_+(x_0(z)) $ into the definition of $ g_+(x_0(z),z) $ yields 
	\begin{equation} \label{eq:gplus}
	g_+(x_0(z),z) = \frac{-\sqrt{i \left(-i z^2+i+2\right)}+(z+1) \left(z 
		\sqrt{i \left(-i z^2+i+2\right)}+i (z-1)^2\right)+2}{\sqrt{2} 
		\left(\sqrt{-i 
			z^2+i+2}+\sqrt{i} z\right)}\,.
	\end{equation}
	Let $ A= \sqrt{-i z^2+i+2}$. Solving $ z $ out of this expression, we get $ 
	z=\pm \frac{\sqrt{-A^2+i+2}}{\sqrt{i}} $. Plugging it into 
	\eqref{eq:gplus}, 
	we obtain \[ 
	g_+(x_0(z),z) = \frac{A \sqrt{i}-A^2+2}{\sqrt{2i} }  \,.
	\]
	Note that the result remains invariant no matter whether we plug in $ z= 
	\frac{\sqrt{-A^2+i+2}}{\sqrt{i}} $ or $ z=- 
	\frac{\sqrt{-A^2+i+2}}{\sqrt{i}} $. We plug in the definition of $ A $ and 
	express $ g_+(x_0(z),z) $ in terms of $ z $ again \[ 
	g_+(x_0(z),z)= 
	\sqrt{\frac{i}{2}}\left[z^2-1+\sqrt{1+\frac{2}{i}-z^2}\right]\,.
	\]
	
	Similarly, since $ \min\{z,z_-(x_0)\} =z_-(x_0)$ and $ a_i(z_-) = 
	\frac{\sqrt{i/2}z_-^2+\sqrt{2/i}}{2} $, plugging $ z_-(x_0(z)) $ into the 
	definition of $ g_-(x_0(z),z) $ yields \begin{equation} \label{eq:gminus}
	g_-(x_0(z),z) = \frac{(z-1) z \sqrt{i \left(-i z^2+i+2\right)}-\sqrt{i 
			\left(-i z^2+i+2\right)}-i (z-1) (z+1)^2+2}{\sqrt{2} \left(\sqrt{-i 
			z^2+i+2}-\sqrt{i} z\right)}\,.
	\end{equation}
	Again plugging $ 
	z=\pm \frac{\sqrt{-A^2+i+2}}{\sqrt{i}} $ into \eqref{eq:gminus} gives 
	\[ 
	g_-(x_0(z),z) = \frac{A \sqrt{i}-A^2+2}{\sqrt{2i} }\,,
	\]
	which equals $ g_+(x_0(z),z) $. Therefore, we conclude that if $ |z|\le 
	\sqrt{2/i} $, \begin{equation}\label{eq:L_small_abs}
	(\cT a_i)(z) = 
	\sqrt{\frac{i}{2}}\left[z^2-1+\sqrt{1+\frac{2}{i}-z^2}\right]\,.
	\end{equation}
	
	Combining \eqref{eq:lz}, \eqref{eq:L_large_abs} and \eqref{eq:L_small_abs}, 
	we establish \[ 
	(\cT a_i)(z) = \begin{cases}
	\sqrt{\frac{i}{2}}\left[z^2-1+\sqrt{1+\frac{2}{i}-z^2}\right],& |z|\le 
	\sqrt{2/i}\,,\\
	|z|,& |z|> \sqrt{2/i}\,.\\
	\end{cases}
	\]
\end{proof}

\subsection{Exact Values of Normalized Minimax Regret}\label{exactvaluesofnormalizedminimaxregret}
Recall that $ u_2(z)=a_2(z) $. \cref{prop:fugal-operator-quadratic} implies \[ 
u_3(z) = (\cT u_2)(z) = (\cT a_2)(z) = \begin{cases}
z^2-1+\sqrt{2-z^2},& |z|\le 
1\,,\\
|z|,& |z|> 1\,.\\
\end{cases}
 \]
 Therefore we have $ u_1(0)=1 $, $ u_2(0)=\frac{1}{2} $, and  $ u_3(0) = 
 \sqrt{2}-1 $. These exact values imply that the minimax regret of a $ T 
 $-round fugal game is exactly $ T $, $ \frac{T}{2} $, and $ 
 (\sqrt{2}-1)T $ if the player is allowed to switch at most $ 0 $, $ 1 $, and $ 
 2 $ 
 times, respectively. 
In \cref{prop:u40}, we compute the exact value of $ u_4(0) $. The complicated 
form of $ u_4(0) $ suggests that it is highly challenging to find a pattern for 
the general form of $ u_i(0) $ and that we should consider lower bounds whose 
behavior under the action of the fugal operator is more amenable to analysis, 
as what we will discuss in \cref{sub:quadratic-lower-bound}.
\begin{prop}\label{prop:u40}
	The value of $ u_4(0) $ is given by \begin{align*} 
& u_4(0)\\
={}& \frac{1}{3} \sqrt[3]{45 \sqrt{2}+3 \sqrt{3 \left(502 
	\sqrt{2}+945\right)}+145}-\frac{5}{3}
	 -\frac{2 \left(3 \sqrt{2}+1\right)}{3 
	\sqrt[3]{45 \sqrt{2}+3 \sqrt{3 \left(502 \sqrt{2}+945\right)}+145}}\\ \approx{}&
	0.362975 \,.
	 \end{align*}
\end{prop}
\begin{proof}
	By the definition of the fugal operator, we have \begin{align*}
	& u_4(0)\\
	={}& (\cT u_3)(0)\\
	 ={}& \inf_{x\in [-1,1]} \max_{w=\pm 1}  
	\inf_{\substack{|z'|<1\\ wz'\ge 0}} 
	\frac{u_3(z')+ x 
		z'}{1+z'w} \\
	={}&  \inf_{x\in [-1,1]} \max_{w=\pm 1}  
	\inf_{\substack{|z'|<1\\ wz'\ge 0}} 
	\frac{(z'^2-1+\sqrt{2-z'^2})+ x 
		z'}{1+z'w}\\
	={}& \inf_{x\in [-1,1]} \max \left\{ 
	\inf_{0\le  z'< 1} 
	\frac{(z'^2-1+\sqrt{2-z'^2})+ x 
		z'}{1+z'},
	\inf_{-1 < z'\le 0} 
	\frac{(z'^2-1+\sqrt{2-z'^2})+ x 
		z'}{1-z'} \right\}  
	 \,.
	 \end{align*}
	If we define $ f(x,z')\triangleq \frac{(z'^2-1+\sqrt{2-z'^2})+ x 
		z'}{1+z'} $ and $ g(x)\triangleq \inf_{0\le z'<1} 
	f(x,z')$, the value of $ u_4(0) $ can be re-written as 
		\begin{align*} 
	u_4(0) ={}& \inf_{x\in [-1,1]} \max \left\{ \inf_{0\le 
	z'<1}f(x,z'),\inf_{-1<z'\le 0} f(-x,-z') \right\}\\
={}& \inf_{x\in [-1,1]} \max \left\{ \inf_{0\le 
	z'<1}f(x,z'),\inf_{0\le z'< 1} f(-x,z') \right\}\\
={}& \inf_{x\in [-1,1]} \max \left\{ g(x),g(-x) \right\} \,.
	 \end{align*}
	 Note that $ f(x,z) $ is non-decreasing with respect to $ x $ provided that 
	 $ z'\in [0,1] $. Therefore, the function $ g(x) $ is non-decreasing in $ x 
	 $ and the inequality $ g(x)\ge g(-x) $ is equivalent to $ x\ge 0 $. As a 
	 result, we deduce that \[ u_4(0)= \inf_{x\in [-1,1]} \max \left\{ 
	 g(x),g(-x) \right\} = g(0)=\inf_{0\le z'<1} f(0,z')= \inf_{0\le z'<1} 
	 \frac{z'^2-1+\sqrt{2-z'^2}}{1+z'}\,. \]
	 The derivative of $ f(0,z') $ with respect to $ z' $ is $ \frac{\partial 
	 f(0,z')}{\partial z'} = -\frac{z'+2}{(z'+1)^2 \sqrt{2-z'^2}}+1 $. Setting 
	 this derivative greater than or equal to $ 0 $ yields a sextic polynomial 
	 $ p(z')\triangleq 
	 -z'^6-4 z'^5-4 z'^4+4 z'^3+10 z'^2+4 z'-2\ge 0 $. By Descartes' rule of 
	 signs, this polynomial has two sign differences and thereby has two 
	 or zero positive roots. Since $ p(0)=-2 $, $ p(1)= 7$ and $ 
	 p(2)=-178 $, we deduce that there is exactly one root in $ (0,1) $ and $ 
	 (1,2) $ respectively. Let $ z_0 $ denote the unique root of $ p(z') $ in $ 
	 (0,1) $. The function $ f(0,z') $ is decreasing on $ [0,z_0] $ and 
	 increasing on $ [z_0,1] $. Thus the desired infimum $ \inf_{0\le z'<1} 
	 f(0,z') $ is attained at $ f(0,z_0) $. 
	 
	 We notice that $ p(z') $ can be factorized in $ \mathbb{Q}(\sqrt{2}) $ as 
	 below \[ 
	 p(z') = \left(z'^3+2 z'^2-\sqrt{2} z'-\sqrt{2}-2\right) \left(z'^3+2 
	 z'^2+\sqrt{2} z'+\sqrt{2}-2\right)\,.
	  \]
	  Solving the two cubic polynomials with Cardano formula, we obtain the 
	  unique root in $ (0,1) $ \begin{align*} 
	  z_0={}& \frac{1}{6} \left(-2 \left(3 \sqrt{2}-4\right) 
	  \sqrt[3]{\frac{2}{-9 
	  \sqrt{2}+3 \sqrt{6 \left(2 \sqrt{2}+9\right)}+38}}\right.\\
	  &\left.+2^{2/3} \sqrt[3]{-9 
	  \sqrt{2}+3 \sqrt{6 \left(2 \sqrt{2}+9\right)}+38}-4\right)\\
  \approx{} &
	  0.283975\,.
	   \end{align*}
	   Plugging $ z'=z_0 $ into $ f(0,z') $ yields the desired expression for $ 
	   u_4(0) $.
\end{proof}
 
\subsection{Interlacing Quadratic Lower Bound Functions}\label{sub:quadratic-lower-bound}

\begin{lemma}[Fugal operator interlaces quadratic lower bound 
functions]\label{lem:fugal-monotone}
	For $ i\ge 1 $ and $ z\in [-1,1] $, $a_{i+1}(z)\le (\cT a_i)(z) $.
\end{lemma}
\begin{proof}
	If $ i=1 $, recall that $ a_1(z)=1 $ and $ a_2(z)=u_2(z)=\frac{z^2+1}{2} $.
	\cref{lem:u2} implies $ a_2(z) = (\cT a_1)(z) $ and therefore the promised 
	inequality holds. In the sequel, we assume that $ i\ge 2 $.
	In 
	\cref{prop:fugal-operator-quadratic}, we show that for $ i\ge 2 $, \[ 
	(\cT a_i)(z) = \begin{cases}
	\sqrt{\frac{i}{2}}\left[z^2-1+\sqrt{1+\frac{2}{i}-z^2}\right],& |z|\le 
	\sqrt{2/i}\,,\\
	|z|,& |z|> \sqrt{2/i}\,.\\
	\end{cases}
	\]
	Recall the definition of $ a_{i+1} $ \[  
	a_{i+1}(z)=\begin{cases}
	\frac{\sqrt{\frac{i+1}{2}}z^2+\sqrt{\frac{2}{i+1}}}{2},&  
	|z|<\sqrt{\frac{2}{i+1}} 
	\,,\\
	|z|,& |z|\ge \sqrt{\frac{2}{i+1}}\,.
	\end{cases}
	\]
	If $ |z|>\sqrt{\frac{2}{i}} $, we have $ (\cT a_i)(z)=a_{i+1}(z) $. If $ 
	\sqrt{\frac{2}{i+1}} \le  |z|\le 
	\sqrt{\frac{2}{i}} $, we need to show that \[ 
	(\cT a_i)(z)-|z|= 
	\sqrt{\frac{i}{2}}\left[z^2-1+\sqrt{1+\frac{2}{i}-z^2}\right] - |z|\ge 0\,.
	\]
	Note that in this case, $ l(z)\triangleq (\cT a_i)(z)-|z| $ is an even 
	function. 
	Therefore it suffices to show the inequality for $ \sqrt{\frac{2}{i+1}} 
	\le z \le 
	\sqrt{\frac{2}{i}} $. For any $ y\in [0,1] $ and $ i\ge 2 $, the 
	following inequality 
	holds \[ 
	\frac{1}{\sqrt{\frac{2}{i}(1-y^2)+1}}+\frac{1}{y}-2 \le \lim_{i\to 
		\infty} \left(\frac{1}{\sqrt{\frac{2}{i}(1-y^2)+1}}+\frac{1}{y}-2 
		\right)
	= \frac{1}{y}-1 \ge 0\,.
	\]
	Since $ \sqrt{\frac{i}{2}}z\in [0,1] $, setting $ y=\sqrt{\frac{i}{2}}z 
	$ in the above inequality gives \[ 
	\frac{1}{\sqrt{\frac{2}{i}-z^2+1}}+\frac{\sqrt{2}}{\sqrt{i} z}-2 \ge 0\,.
	\]
	Re-arranging the terms, we get \[ 
	\frac{dl}{dz}= -1+ \sqrt{\frac{2}{i}}z\left( 
	2-\frac{1}{\sqrt{\frac{2}{i}-z^2+1}} 
	\right)\le 0\,.
	\]
	This implies that $ l(z) $ is non-increasing if $ z\le \sqrt{\frac{2}{i}} 
	$. Therefore, for any $ \sqrt{\frac{2}{i+1}} 
	\le z \le 
	\sqrt{\frac{2}{i}} $, we have $ l(z)\ge l(\sqrt{\frac{2}{i}})=0 $.
	
	If $ |z|\le \sqrt{\frac{2}{i+1}} $, we need to show that \[ 
	(\cT 
	a_i)(z)-a_{i+1}(z)=\sqrt{\frac{i}{2}}\left[z^2-1+\sqrt{1+\frac{2}{i}-z^2}\right]
	- 
	\frac{\sqrt{\frac{i+1}{2}}z^2+\sqrt{\frac{2}{i+1}}}{2} \ge 0 \,.
	\]
	Since in this case the function $ (\cT a_i)(z)-a_{i+1}(z) $ is an even 
	function 
	with respect to $ z $, we assume that $ 0\le z\le \sqrt{\frac{2}{i+1}} $.
	Since $ (\cT a_i)(z)-a_{i+1}(z) $ is a concave function with respect to $ 
	z^2 $ 
	(note 
	that $ \sqrt{1+\frac{2}{i}-z^2} $ is concave with respect to $ z^2 $ and 
	that 
	the remaining terms are linear in $ z^2 $), it is 
	sufficient to check its non-negativity when $ z^2=0 $ and $ 
	z^2=\frac{2}{i+1} $ 
	(\ie, when $ z=0 $ and $ z=\sqrt{\frac{2}{i+1}} $).
	Recall that we have shown that $ (\cT a_i)(z)-a_{i+1}(z)\ge 0 $ holds for 
	any $ 
	\sqrt{\frac{2}{i+1}}\le z\le \sqrt{\frac{2}{i}} $. It remains to check the 
	non-negativity of $ (\cT a_i)(0)-a_{i+1}(0) $. We have $ (\cT 
	a_i)(0)-a_{i+1}(0) = 
	\frac{1}{\sqrt{2}}(-\sqrt{i}+\sqrt{i+2}-\frac{1}{\sqrt{i+1}}) $. The 
	concavity 
	of the square root function implies $ \sqrt{i+1}\ge 
	\frac{\sqrt{i+2}+\sqrt{i}}{2} =\frac{1}{\sqrt{i+2}-\sqrt{i}} $. Thus we 
	obtain 
	$ \sqrt{i+2}-\sqrt{i}\ge \frac{1}{\sqrt{i+1}} $ and the non-negativity of $ 
	(\cT a_i)(0)-a_{i+1}(0) $. We conclude that $ (\cT a_i)(z)\ge a_{i+1}(z) $ 
	in 
	all three 
	cases.
\end{proof}

\cref{lem:alt} shows that the quadratic lower bound functions indeed provide a 
lower bound for $ u_k(z) $.
\begin{lemma}[Quadratic lower bound]\label{lem:alt}
	 For all $ k\ge 
	 1 $ and $ |z|<1 $, it holds that $ a_k(z)\le u_k(z) $.
\end{lemma}
\begin{proof} %
	If $ k=1 $, the claim holds by recalling $ u_k(z)=1 $ on $ (-1,1) $, as 
	shown in \cref{cor:u1}. If $ k=2 $, we have $ a_2(z)=\frac{z^2+1}{2}\le 
	u_2(z) $ by \cref{lem:u2}, 
	as promised. For $ k>2 $, we will show the desired inequality by induction. 
	Assume that $ a_i(z)\le u_i(z) $ for some $ i\ge 2 $. We will show that $ 
	a_{i+1}(z)\le u_{i+1}(z) $. Since $ a_i(z)\le u_i(z) $, by 
	\cref{lem:improved-recursive-uk}, we deduce \begin{align*}
	u_{i+1}(z) ={}&
	 \inf_{x\in [-1,1]} \max_{w=\pm 1}  
	\inf_{\substack{|z'|<1\\ w(z'-z)\ge 0}} 
	\frac{(1+wz) u_i(z')+ x 
		(z'-z)}{1+z'w} \\
	\ge{} &  \inf_{x\in [-1,1]} \max_{w=\pm 1}  
	\inf_{\substack{|z'|<1\\ w(z'-z)\ge 0}} 
	\frac{(1+wz) a_i(z')+ x 
		(z'-z)}{1+z'w}
	=(\cT a_i)(z) \stackrel{(a)}{\ge} a_{i+1}(z)
	 \,,
	 \end{align*}
where $ (a) $ is due to \cref{lem:fugal-monotone}.
 By induction, 
we know that $ a_k(z)\le u_k(z) $ holds for all $ k\ge 1 $ and $ |z|<1 $.

\end{proof}

We are in a position to prove the minimax lower bound for the one-dimensional 
game.
\begin{proof}[Proof of \cref{thm:1-d}]
	By \cref{lem:alt}, plugging $ z=0 $ into $ u_K(z)\ge a_K(z) $ shows that 
	the normalized minimax regret without initial bias $ u_K(0) \ge 
	a_K(0) = \frac{1}{\sqrt{2K}} $. 
	Recall that $ u_K(0) 
	=  \frac{1}{T}r_K(T,0) 
	$ for all $ T>0 $, where $ r_K(T,0) $  is the minimax regret with $ T $ 
	rounds, a maximum 
	number of $ K $ switches, and without initial bias. Therefore, 
	 we have $ \frac{1}{T}r_K(T,0)\ge 
	\frac{1}{\sqrt{2K}}  $, 
	which implies that $ R_K(T,0)\ge \frac{T}{\sqrt{2K}} $
	 (because the minimax regret of switching-constrained 
	online convex optimization is lower bounded by the minimax regret of a 
	fugal game).
\end{proof}

\subsection{Tightness of Lower Bound}\label{sub:tight}
In the following two propositions, we validate the one-dimensional lower bound of the previous section in two senses. First, in \cref{prop:unimprovable-constant} we show that the constant in \cref{thm:1-d} cannot be increased for arbitrary $K$ and $T$. In particular, we demonstrate that when $K=2$, the player has a simple strategy --- playing 0 in the first half of the rounds, and an appropriately chosen constant in the second half --- to guarantee regret no greater than $\lceil T/2\rceil$.
\propunimprovableconstant*
\begin{proof}
	We will show that the lower bound is tight when $ K=2 $ by proving the 
	upper bound $ \roco(B^1, B^{*1}, 2, T) \le \lceil T/2\rceil $. Recall that 
	if $ K=2 $, the lower bound $ \frac{T}{\sqrt{2K}} $ is $ T/2 $. To prove 
	the upper bound, we consider the following player's strategy. First, we 
	assume that $ T $ is an even number and we will address the situation where 
	$ T $ is odd later. The player plays $ 0 $ in the first half of the rounds. 
	Let $ W_1 $ be the sum of numbers that the adversary plays in the first 
	half of the rounds and $ W_2 $ be the sum in the second half. In other 
	words, $ W_1=\sum_{t=1}^{T/2} w_t $ and $ W_2=\sum_{t=T/2+1}^{T} w_t $. In 
	the second half of the rounds, the player plays $ -\frac{W_1}{T/2} $. Since 
	$ |W_1| \le T/2 $, the player's choice $ -\frac{W_1}{T/2} $ lies in $ 
	[-1,1] $. The regret is equal to \[ 
	W_2\cdot \left(-\frac{W_1}{T/2}\right) + |W_1+W_2|\,.
	 \]
	If $ W_1+W_2 $ is non-negative, the regret equals $ 
	W_1+W_2-\frac{2W_1W_2}{T} = 
	W_1+W_2(1-\frac{2W_1}{T}) \le W_1+\frac{T}{2}(1-\frac{2W_1}{T})= 
	\frac{T}{2} $, 
	where the inequality is because $ 1-\frac{2W_1}{T}\ge 0 $ and $ W_2 \le 
	\frac{T}{2} $. If $ W_1+W_2 $ is negative, the regret becomes $ 
	-W_1-W_2-\frac{2W_1W_2}{T} = -W_1-W_2(1+\frac{2W_1}{T})\le 
	-W_1+\frac{T}{2}(1+\frac{2W_1}{T})=\frac{T}{2} $, where the inequality is 
	because $ 1+\frac{2W_1}{T}\ge 0 $ and $ W_2\ge -\frac{T}{2} $. Therefore, 
	the regret is at most $ \frac{T}{2} $. If $ T $ is odd, the player plays $ 
	0 $ at the first round and the number of remaining rounds is $ T-1 $, which 
	is even. The player then uses the previous strategy for an even $ T $. In other words, 
	the player plays $ 0 $ from the first round to the $ \frac{T+1}{2} $-th 
	round and plays $ -\frac{\sum_{t=2}^{(T+1)/2}w_t}{(T-1)/2} $ at all 
	remaining rounds. The regret differs from the regret in the $ (T-1) $-round 
	game by at most $ 1 $. Therefore, the regret is upper bounded by $ 
	\frac{T-1}{2}+1=\frac{T+1}{2}=\lceil \frac{T}{2}\rceil $.
\end{proof}

The previous proposition demonstrated that the constant $\frac{1}{\sqrt{2}}$ could not be improved when $K=2$, and thus could not be increased for an arbitrary $K$. In the next proposition, we show that our previous analysis of the fugal game was ``tight" in a separate, asymptotic sense. When $K=o(T)$, the minimax regret of the fugal game relaxation is asymptotically (in $T$) equal to that of the original, switching-constrained OCO formulation. To understand the implication of this result, recall that the fugal game departed from the original game in two key ways. First, the player was permitted to choose non-discrete block lengths, $M_i \geq 0$, rather than only integral $M_i$. It is perhaps unsurprising that, as $T$ grows large, this restriction does not make a difference: intuitively, one can approximate $\frac{M_i}{T}$, where $M_i$ is non-integral and $T$ is small, arbitrarily well by $\frac{\tilde{M}_i}{\tilde{T}}$, where $\tilde{M}_i$ is integral but both it and $\tilde{T}$ are large. However, the fugal game also required the adversary to copy the player's switching pattern, and to play only $\pm1$. It may be surprising that the combination of these various restrictions has no affect on the minimax rate, asymptotically.

To prove the result, we present a reduction which converts the player's optimal algorithm in achieving the fugal minimax rate, to an algorithm (\cref{alg:fugal-for-oco}) for ordinary, switching-constrained OCO. The regret of this algorithm against an optimal adversary necessarily upper bounds the constrained OCO minimax rate by \cref{sec:1-d}. Intuitively, the player simulates a fugal game based on the real game, and chooses actions based on the simulated game. The player's strategy in \cref{alg:fugal-for-oco} ``translates" in an appropriate manner from the actual game to a simulated fugal game, and proceeds according to the optimal strategy in the simulated game. In particular, she converts from the received, non-integral $w_i$ to an internal, stored set of fugal $w'_i \in \{ \pm 1 \}$, representing the closest approximation to a fugal game of the actual game. Once the adversary's cumulative action since the last switch, $W_t$, exceeds (in absolute value) the equivalent quantity in the fugal game, the player switches actions. She consults the fugal strategy as an oracle to pick which action to play, and the game continues. By some algebraic manipulation, we show that the regret of the ``simulated" fugal game, and the real game, stay reasonably close. We can thus upper bound the ordinary minimax rate in terms of the fugal minimax rate and an additive term which disappears in the limit of $T$, obtaining the stated result.

\begin{prop}[Asymptotic tightness of fugal lower 
bound]\label{prop:asymptotic-tightness}
	For any fixed $ K\ge 1 $, we have the limit $ \lim_{T\to \infty} 
	\frac{1}{T}\roco(B^1, 
	B^{*1}, 
	K, T) = u_K(0)  $, where $ u_K(0) $ is defined in \cref{lem:independent_T} 
	and denotes the normalized minimax regret with no initial bias.
\end{prop}
\begin{proof}
	Let $ 1=m_1 < m_2 <\cdots < m_{K_T} $ denote all moving rounds.
	 For any integer $ 1< t\le T $, let $ p(t) $ 
	be 
	the largest integer such that $ m_{p(t)} < t $. 
	Recall the regret of a $ T $-round fugal game with a maximum number of $ 
	k-1 $ switches and no initial bias is given by \[ 
	\sup_{\lambda>0} \left(  
	\sum_{i=1}^k 
	M_iw_i x_i + \left|
	\sum_{i=1}^k 
	M_iw_i\right| + \lambda \ind[\sum_{i=1}^k M_i\ne T] \right)\,.
	 \]
	Let $ x_i^*(w_1,\dots,w_{i-1}):\{-1, 1\}^{i-1}\to [-1,1] $ and $ 
	M_i^*(w_1,\dots,w_i):\{-1, 1\}^i \to \bR_{\ge 0} $ be the optimal strategy 
	of the player in the fugal game, where $ i=1,\dots,K $. We will use this strategy to 
	construct a player' strategy 
	for the switching-constrained OCO, which is presented in 
	\cref{alg:fugal-for-oco}. 
		\begin{algorithm}[htb]
		\caption{Player's strategy for 
			switching-constrained OCO derived from fugal 
			games\label{alg:fugal-for-oco}}
		\Output{Player's moves $ x_1,\dots,x_T $.}
		\For{$ t=1,\dots,T $\label{ln:for-loop}}{
			Observe $ w_{t-1} $\;
			\eIf{$ t=1 $}{
				$ K_1\gets 1 $\;
				Play $ x_1 \gets x_1^* $\;}{
				$ W_t\gets \sum_{j=m_{p(t)}}^{t-1} w_j $\;
				$ U_{p(t)}\gets M_{p(t)}^*(w'_1,w'_2,\dots,w'_{p(t)-1},1) 
				$\;
				$ L_{p(t)}\gets -M_{p(t)}^*(w'_1,w'_2,\dots,w'_{p(t)-1},-1) $\;
				\uIf{$ W_t\ge U_{p(t)} 
					$}{$ K_t\gets K_{t-1}+1 $\;
					$ w'_{p(t)}\gets 1 $\;
					Play $ x_t\gets x_{p(t)+1}^*(w'_1,w'_2,\dots,w'_{p(t)}) $\;
				}
				\uElseIf{$ W_t\le L_{p(t)}$ 
				}{$ K_t\gets K_{t-1}+1 $\;
				$ w'_{p(t)}\gets -1 $\;
					Play $ x_t\gets x_{p(t)+1}^*(w'_1,w'_2,\dots,w'_{p(t)}) $\;
				}\Else{$ K_t\gets K_{t-1} $\;
					Play $ x_t\gets x_{t-1} $\;}
			}
		}
	\end{algorithm}
	
	First, we claim that $ K_T\le K $. According to the algorithm, the 
	instruction $ K_t\gets 
	K_{t-1}+1 $ is executed when $ t>1 $ and either $ W_t\ge U_{p(t)} $ or $ 
	W_t\le L_{p(t)} $ happens. In 
	both cases, 
	we have $ |W_t|\ge M^*_{p(t)}(w'_1,w'_2,\dots,w'_{p(t)}) $. Since the $ t 
	$-th round is a moving round if the instruction $ K_t\gets K_{t-1}+1 $ 
	is executed, we 
	get $ 
	m_{p(t+1)}=t $. Since $ |W_t|\le 
	t-1-m_{p(t)}+1=t-m_{p(t)}=m_{p(t+1)}-m_{p(t)} 
	$, the 
	inequality $ m_{p(t+1)}-m_{p(t)} =m_{K_t}-m_{K_t-1}  \ge 
	M^*_{p(t)}(w'_1,w'_2,\dots,w'_{p(t)}) $ must 
	hold. Note that the above equality is true only if the $ t $-th round is a 
	moving round. Additionally, notice that for any $k$, $ K_{m_k}=k $ and $ p(m_k)=k-1 
	$. %
	If $ K_T\ge K+1 $, summing the inequality over all $ t\in \{m_k|2\le 
	k\le K+1 \} 
	$ 
	yields \begin{align*}
&	\sum_{k=2}^{K+1} (m_{K_{m_k}} -  m_{K_{m_k}-1}) =\sum_{k=2}^{K+1} (m_{k} -  
	m_{k-1})= m_{K+1}-1\\  \ge{}& \sum_{k=2}^{K+1} 
	M^*_{p(m_k)}(w'_1,w'_2,\dots,w'_{p(m_k)})=\sum_{k=2}^{K+1} 
	M^*_{k-1}(w'_1,w'_2,\dots,w'_{k-1}) = T\,,
	 \end{align*}
	where the last equality is because for any given sequence $ 
	w'_1,w'_2,\dots,w'_K $, the sum $ \sum_{k=1}^K 
	M^*_k(w'_1,w'_2,\dots,w'_{k}) $ must be $ T $. Since we assume $ K_T\ge K+1 
	$, we deduce $ T\ge m_{K_T}\ge m_{K+1}\ge T+1 $, which is a contradiction. 
	Therefore, we establish $ K_T\le K $.
	
	Since $ K_T\le K $, for the purpose of analysis, let us modify 
	\cref{ln:for-loop} and wait until $ K_t=K+1 
	$. In other words, the algorithm terminates at the $ T_0 $-th round 
	whenever 
	$ K_{T_0}=K+1 $ happens. We define $ m_{K+1}=T_0 $. The algorithm continues 
	running even if $ t>T $, 
	provided that $ K_t\le K $. We define $ T'=T_0-1\ge T $. The $ T' $-th 
	round is the last round such that $ K_{T'}=K $. Note that in the following calculations, $x_t$ and $w'_t$ refer to the assignments made in \cref{alg:fugal-for-oco}. Since the adversary can 
	always play $ 0 $ at the additional rounds (\ie, $ w_{T+1}=w_{T+2}=\dots = 
	w_{T'}=0 $), we have \begin{align*}
&	\max_{w_1,\dots,w_T} \sum_{t=1}^T x_t \cdot w_t + \left| \sum_{t=1}^T w_t 
	\right| \le \max_{w_1,\dots,w_{T'}} \sum_{t=1}^{T'} x_t \cdot w_t + \left| 
	\sum_{t=1}^{T'} w_t 
	\right| \\
	={}& \max_{w_1,\dots,w_{T'}} \sum_{i=1}^K 
	x_{i}^*(w'_1,w'_2,\dots,w'_{i-1})\sum_{j=m_i}^{m_{i+1}-1} w_j + \left| 
	\sum_{i=1}^K \sum_{j=m_i}^{m_{i+1}-1} w_j \right|\,.
	 \end{align*}
	 For any $ 1\le i\le K $, if $ w'_i=1 $, since $ |w_t|\le 1 $ for all $ t 
	 $, we have $0\le  \sum_{j=m_i}^{m_{i+1}-1} 
	 w_j-M_i^*(w'_1,w'_2,\dots,w'_i)\le 1 $. If $ w'_i=-1 $, similarly we get $ 
	 0 \le -\sum_{j=m_i}^{m_{i+1}-1} 
	 w_j-M_i^*(w'_1,w'_2,\dots,w'_i)\le 1 $. Combining these two cases gives $ 
	 0\le w'_i\sum_{j=m_i}^{m_{i+1}-1} 
	 w_j-M_i^*(w'_1,w'_2,\dots,w'_i)\le 1  $. Multiplying through by $w_i'$, we therefore obtain \[\left| 
	 \sum_{j=m_i}^{m_{i+1}-1} 
	 w_j-w'_i M_i^*(w'_1,w'_2,\dots,w'_i) \right|\le 1\,.
	 \] 
	 Thus the following upper bound holds \begin{align*}
&\roco(B^1, 
B^{*1}, 
K, T) \\
\le {}&	\max_{w_1,\dots,w_T} \sum_{t=1}^T x_t \cdot w_t + \left| 
\sum_{t=1}^T 
	 w_t 
	 \right| \\
\le{}& \max_{w_1,\dots,w_{T'}} \sum_{i=1}^K 
	 x_{i}^*(w'_1,w'_2,\dots,w'_{i-1})\sum_{j=m_i}^{m_{i+1}-1} w_j + \left| 
	 \sum_{i=1}^K \sum_{j=m_i}^{m_{i+1}-1} w_j \right| \\
	 \le{}& \max_{w_1,\dots,w_{T'}} \sum_{i=1}^K 
	 x_{i}^*(w'_1,w'_2,\dots,w'_{i-1})w'_i M^*_i(w'_1,w'_2,\dots,w'_i) + \left| 
	 \sum_{i=1}^K w'_i M^*_i(w'_1,w'_2,\dots,w'_i) \right| +2K\\
	 ={}& r_K(T,0)+2K \,,
	  \end{align*}
	where $ r_K $ is defined in \eqref{eq:def_rk} and denotes the minimax 
	regret of a $ T $-round fugal game with a maximum number of $ K-1 $ 
	switches and no initial bias. Recalling \cref{lem:independent_T} yields \[ 
	\frac{1}{T} \roco(B^1, 
	B^{*1}, 
	K, T) \le \frac{1}{T} r_K(T,0)+2K=u_K(0) + \frac{2K}{T}\,.
	 \]
	Since the fugal game provides a lower bound for $ \roco(B^1, 
	B^{*1}, 
	K, T) $, it follows that \[\frac{1}{T}\roco(B^1, 
	B^{*1}, 
	K, T)\ge \frac{1}{T}r_K(T,0)= u_K(0) \,.\] 
	 By assumption, $\lim_{T \rightarrow \infty}\frac{2K}{T}=0$. Thus the limit 
	 $ \lim_{T\to \infty} \frac{1}{T} 
	\roco(B^1, 
	B^{*1}, 
	K, T) $ exists and equals $ u_K(0) $.
\end{proof}

\section{Additional Lower Bounds for Higher-Dimensional Switching-Constrained OCO}\label{appendix:sec:high-d}

\cref{prop:infty-norm-high-d} proves a dimension-dependent lower bound. In other words, the problem is harder as the dimension becomes higher. The proof is based on \cref{thm:1-d}
and the observation that if both the player and the adversary select from the $\infty$-norm unit ball, the problem can be decomposed into $n$ fully decoupled one-dimensional sub-problems.

\begin{prop}[Lower bound for $\infty$-norm]\label{prop:infty-norm-high-d}
    The minimax regret $\roco(B^{n}_\infty,B^{*n}_\infty,K,T)$ is at least 
    $\frac{nT}{\sqrt{2K}}$.
\end{prop}
\begin{proof}
    By \eqref{eq:regret-p-norm}, we have \begin{align*}
    \roco(B^{n}_\infty,B^{*n}_\infty,K,T)
     ={} \inf_{\|x_1\|_\infty\le 1} & \sup_{\|w_1\|_\infty\le 1} 
\dots 
\inf_{\|x_T\|_\infty\le 1
}\sup_{\|w_T\|_\infty\le 1} \sup_{\lambda>0}\\
&\left(  \sum_{i=1}^T w_i \cdot x_i 
+\left\|
\sum_{i=1}^T 
w_i\right\|_{1} + \lambda \ind[c(x_1,\dots,x_T)\ge K] \right)\,.
    \end{align*}
Both terms are decomposable by coordinates as follows: the $j^{th}$ coordinate 
of $\sum_{i=1}^T w_i \cdot x_i  =\sum_{j=1}^n w_{i,j} x_{i,j}$ and $\left\| 
\sum_{i=1}^T w_i \right\|= \sum_{j=1}^n \left| \sum_{i=1}^T w_{i,j} \right|$. 
Therefore by \cref{thm:1-d}, we obtain \[\roco(B^{n}_\infty,B^{*n}_\infty,K,T)= 
n \roco(B^1,B^{*1},K,T)\ge \frac{nT}{\sqrt{2K}}\,.\]
\end{proof}

\section{Additional Upper Bounds for Switching-Constrained OCO} \label{appendix:upper_bounds_OCO}
In this section, we derive upper bounds for switching-constrained OCO to match 
the lower bounds in \cref{prop:1d-lowerbound,thm:high-d}. We begin with a 
simple algorithm achieving the correct minimax regret, $O(\frac{T}{\sqrt{K}})$, 
for any player's action set $ \cD $ and the function family $ \cF $ that the 
adversary chooses from. 

\propupperboundocohighd*
 	\begin{proof}
 First, we claim that the minimax regret $ \roco(\cD,\cF, K, T) $ is a 
 non-decreasing function in $ T $. To see this, consider the situation where we 
 have more rounds. The adversary can play $ 0 $ in all additional rounds and 
 this does not decrease the regret. Therefore, we obtain that $ \roco(\cD, \cF, 
 K, T) \le \roco(\cD,\cF, K, T_1) $, where $ T_1 = \lceil 
 \frac{T}{K} \rceil K \ge T $.

In the sequel, we derive an upper bound for $ \roco(\cD, \cF, K, T_1) $. 
 To attain the upper bound, we mini-batch the $ T_1 $ 
 	rounds into $ K $ equisized epochs, each having size
 	$  \frac{T_1}{K} = \lceil \frac{T}{K} \rceil  $. Let $ E_i $ denote 
 	the set of all rounds that belong to the $ i $-th epoch. We have $ 
 	 E_i = \{ \frac{T_1}{K}(i-1)+1,\frac{T_1}{K}(i-1)+2,\dots,\frac{T_1}{K}i \} 
 	 $.
 	 The 
 	epoch loss of the $ i $-th epoch is the average of loss functions in this 
 	epoch, \ie, $ \bar{f}_i\triangleq  \frac{1}{|E_i|} 
 	\sum_{j\in E_i} f_j $. If we run a minimax 
 	optimal algorithm for unconstrained OCO (for example, online gradient 
 	descent \cite{zinkevich2003online}) on the epoch losses $ 
 	\bar{f}_1,\dots,\bar{f}_{K} $ and obtain the player's action sequence $ 
 	\bar{x}_1,\dots,\bar{x}_{K} $, our strategy is to play $ \bar{x}_i $ at all 
 	rounds in the $ i $-th epoch. 
 	This method was originally discussed in \cite{arora2012online,dekel2012optimal}. Using this 
 	mini-batching method, we deduce that the regret is upper bounded by $ 
 	\frac{T_1}{K} O(\sqrt{K})  = \lceil \frac{T}{K} \rceil O(\sqrt{K}) = 
 	O(\frac{T}{\sqrt{K}}) $, where 
 	$O(\sqrt{K})$ is the standard upper bound of the regret of a $K$-round OCO 
 	game. 
	\end{proof}
	
	In the next two propositions, we seek a more precise understanding of the 
	\emph{exact} minimax rate --- \ie\ the constant in front of 
	$\frac{T}{\sqrt{K}}$ --- of switching-constrained online linear 
	optimization, beginning with $n=1$. 
	In \cref{sec:1-d}, \cref{prop:unimprovable-constant} demonstrated that we 
	cannot hope to improve the constant in the lower bound, 
	$\frac{1}{\sqrt{2}}$, for arbitrary $T$ and $K$. Further, 
	\cref{prop:asymptotic-tightness} showed that the fugal game captures the 
	correct constant, \emph{asymptotically}. In the following proposition, we 
	seek a direct, \emph{non-asymptotic} bound on the constant in front of 
	the one-dimensional minimax rate, $\tilde{O}(\frac{T}{\sqrt{K}})$. To do 
	so, we more carefully examine the mini-batching technique from 
	\cref{prop:upper_bound_oco_highd}. We observe that it actually allows reuse 
	of the exact \emph{minimax} rate (including the constant) of vanilla 
	unconstrained OCO, rather than simply algorithms like projected gradient 
	descent in our original application of the technique. %
\begin{lemma}\label{lem:t-k-inequality}
	If $ T $ and $ K $ are positive integers such that $ T\ge K\ge 1 $, the 
	inequality $ \lceil \frac{T}{K} \rceil \le \frac{2T}{\sqrt{K(K+1)}} $ holds.
\end{lemma}
\begin{proof}
	If $ T $ is divisible by $ K $, we have $ \lceil \frac{T}{K} \rceil = 
	\frac{T}{K} $. Since $ \sqrt{\frac{K+1}{K}} \le 2 $, we get $ \frac{1}{K} 
	\le \frac{2}{\sqrt{K(K+1)}} $ and thus $ \lceil \frac{T}{K} \rceil =  
	\frac{T}{K} \le \frac{2T}{\sqrt{K(K+1)}} $. In the sequel, we assume that $ 
	K $ cannot divide $ T $. 
	We consider the Euclidean division of $ T $ by $ K $. There exists unique 
	positive 
	integers $ q $ and $ r $ such that $ T=qK+r $ and $ 1\le r\le K $. The 
	following inequality holds \begin{equation}\label{eq:qr}
	\frac{q+1}{q+r/K} \le \frac{q+1}{q+1/K} \le \frac{2}{1+1/K} = 2\cdot 
	\frac{K}{K+1} \le 2\sqrt{\frac{K}{K+1}}\,,
	 \end{equation}
	where the first inequality is because $ \frac{q+1}{q+r/K} $ is a decreasing 
	function in $ r $ and the second inequality is because $ \frac{q+1}{q+1/K} 
	$ is a decreasing function in $ q $. In light of \eqref{eq:qr}, we have 
	\[ 
	\left\lceil \frac{T}{K}\right\rceil = q+1 \le 2(q+r/K)\sqrt{\frac{K}{K+1}} 
	= \frac{2(Kq+r)}{\sqrt{K(K+1)}} = \frac{2T}{\sqrt{K(K+1)}}\,.
	 \]
\end{proof}

\begin{prop}[Upper bound for $ 2 $-norm and dimension at least 
	2]\label{prop:exact_upper_bound_3d}
	The minimax regret 
	$ \roco(B^n_2, B^{*n}_2, K, T) $ satisfies  $ \roco(B^n_2, B^{*n}_2, K, T) 
	\le \left\lceil\frac{T}{K}\right\rceil \sqrt{K}\le  \frac{2T}{\sqrt{K}} $ 
	for 
	all $n \geq 2 $.
\end{prop}
\begin{proof}
	Let $ \cR(T) $ denote the minimax regret of the vanilla $ T $-round 
	$ n $-dimensional 
	OCO without a switching constraint. It is defined by \[ 
	\cR(T) = \inf_{x_1\in B^n_2} \sup_{w_1\in B^{*n}_2} \dots \inf_{x_T\in 
		B^n_2}\sup_{w_T\in B^{*n}_2} \left( \sum_{i=1}^T w_i \cdot x_i + 
		\left\| 
	\sum_{i=1}^T w_i \right\| \right)\,.
	\]
	Using the mini-batching argument that we used to show 
	\cref{prop:upper_bound_oco_highd}, we have $ \roco(B^n, 
	B^{*n}, 
	K, T) \le \lceil\frac{T}{K}\rceil \cR(K) $. \footnote{ To be concrete, the 
	minimax regret, $\roco(B^1, 
		B^{*1}, 
		K, T)$ can only increase when we restrict the player to switch 
		precisely every $\frac{T}{K}$ rounds. Then, conditioned on this player 
		strategy, the regret term $\sum w_tx_t+ \left| \sum w_t \right|$ is 
		unchanged by forcing the adversary to also pick the same function on 
		each $\frac{T}{K}$-sized block. Thus, $\frac{T}{K} \cR(K)$ provides a 
		valid upper bound as claimed.}
	 By Theorem 6 of 
	\citet{abernethy2008optimal}, $\cR(K) = \sqrt{K}$ when $n>2$. In fact, when 
	$n=2$, the upper bound of Lemma 9 carries through, so $\cR(K) = \sqrt{K}$ 
	when 
	$n=2$ as well. Thus by \cref{lem:t-k-inequality}, we have \[\roco(B^n_2, 
	B^{*n}_2, K, T) \leq \left\lceil\frac{T}{K}\right\rceil\cR(K) 
	= \left\lceil\frac{T}{K}\right\rceil \sqrt{K}
	\le 
	\frac{2T}{\sqrt{K(K+1)}}\sqrt{K} = \frac{2T}{\sqrt{K+1}} < 
	\frac{2T}{\sqrt{K}} 
	\,.\]
\end{proof}

\begin{prop}\label{prop:non-decrease-in-n}
	For any $p$ and $q$ such that $1\le p,q\le \infty$, the minimax regret 
	$\roco(B^n_p,B^{*n}_q,K,T)$ is non-decreasing in the dimension $n$.
\end{prop}
\begin{proof}
	We will show that for any $m<n$, it holds that $\roco(B^m_p,B^{*m}_q,K,T) 
	\le \roco(B^n_p,B^{*n}_q,K,T)$. We view $B^m_p$ ($B^{*m}_q$, respectively) 
	as the subset of $B^n_p$ ($B^{*n}_q$, respectively) by setting the last 
	$n-m$ coordinates to 0. Next, we show how to convert a minimax optimal 
	player's strategy in the $n$-dimensional game into a player's strategy in 
	the $m$-dimensional game. 
	Let $x_i^*(x_1,w_1,\dots,x_{i-1},w_{i-1}):B^n_p\times B^{*n}_q\times \dots 
	\times B^n_p\times B^{*n}_q\to B^n_p$ be the optimal strategy of the player 
	in the $n$-dimensional game.
	Note that any adversary's choice $w_t\in B^{*m}_q$ can be viewed as a 
	choice in $B^{*n}_q$. At the $t$-th round of the $m$-dimensional game, 
	given the adversary's previous choices $w_1,\dots,w_{t-1}$ and the player's 
	previous choices $x_1,\dots,x_{t-1}$,  
	the player computes $x'_t = x_t^*(x_1,w_1,\dots,x_{t-1},w_{t-1})$ and plays 
	$x_t =P(x'_t)$, where $P$ is the orthogonal projection onto $B^m_p$ (\ie, 
	setting the last $n-m$ coordinates to 0). Notice that $w_t\cdot x_t = 
	w_t\cdot 
	x'_t$. Therefore, in light of \eqref{eq:regret-p-norm}, the regret of the 
	$m$-dimensional game $\sum_{t=1}^T w_t\cdot x_t + \|\sum_{t=1}^T 
	w_t\|_{p/(p-1)}$ equals the regret of the $n$-dimensional game 
	$\sum_{t=1}^T 
	w_t\cdot x'_t + \|\sum_{t=1}^T w_t\|_{p/(p-1)}$, and is thus at most 
	$\roco(B^n_p,B^{*n}_q,K,T)$.
\end{proof}

\begin{prop}[Upper bound for one dimension]\label{prop:upper-bound}
	The minimax regret $ \roco(B^1, 
	B^{*1}, 
	K, T) $ satisfies
	\begin{compactenum}[(a)]
		\item For all $K\ge 1$, $ \roco(B^1, 
		B^{*1}, 
		K, T) \le \left\lceil\frac{T}{K}\right\rceil \min\{ 
		\sqrt{\frac{2(K+1)}{\pi}},  \sqrt{K} \}  
		\le 
		2\sqrt{\frac{2}{\pi}}\frac{T}{\sqrt{K}} < \frac{1.6T}{\sqrt{K}} $; 
		and\label{it:first-1d-upper-bound}
		\item For all $ K\ge 2 $, $ \roco(B^1, 
		B^{*1}, 
		K, T) \le \frac{\sqrt{3}}{2} \lceil \frac{T}{K} \rceil \sqrt{K} < 0.87 
		\lceil \frac{T}{K} \rceil \sqrt{K}
		$.\label{it:second-1d-upper-bound}
	\end{compactenum}
\end{prop}
\begin{proof}
	As in \cref{prop:exact_upper_bound_3d}, a more careful inspection of the 
	mini-batching 
	argument reveals that the minimax regret $\roco(B^1, B^{*1}, K, T)$ is 
	at most $\lceil\frac{T}{K}\rceil$ 
	times $\cR(K)$, the minimax rate of vanilla OCO.
	 If $ K $ is even, Theorem~10 of 
	 \citep{mcmahan2013minimax} implies that $\cR(K) = 
	 \frac{K}{2^K}\binom{K}{\frac{K}{2}} \le \sqrt{\frac{2K}{\pi}}$.
	 \citet{mcmahan2013minimax} did not report the minimax regret if $ K $ 
	 is odd. If $ K $ is odd, according to (10) of \citep{mcmahan2013minimax}, 
	 we have \[ 
	 \cR(K) = \frac{1}{2^K} \sum_{i=0}^K \binom{K}{i} |2i-K| = \frac{4}{2^K} 
	 \sum_{i=0}^{(K-1)/2} \binom{K}{i} (\frac{K}{2}-i)\,.
	  \]
	  The minuend equals \[ 
	  \sum_{i=0}^{(K-1)/2} \binom{K}{i} \frac{K}{2} = \frac{K}{2}\cdot 
	  \frac{2^K}{2} = \frac{K 2^K}{4}\,.
	   \]
	   The subtrahend is given by \[ 
	   \sum_{i=0}^{(K-1)/2} \binom{K}{i} i = K\sum_{i=1}^{(K-1)/2} 
	   \binom{K-1}{i-1} = K\sum_{i=0}^{(K-3)/2}  
	   \binom{K-1}{i} = 
	   \frac{K}{2}\left(2^{K-1}-\binom{K-1}{\frac{K-1}{2}}\right)\,.
	    \]
	    Putting them together yields \[ 
	   \cR(K) = \frac{K}{2^{K-1}}\binom{K-1}{\frac{K-1}{2}}\,. 
	    \]
	  Next, we verify that if $K$ is odd, $\cR(K) = \cR(K+1) $. We have
	  \begin{align*}
	  	 \cR(K+1) ={}& \frac{K+1}{2^{K+1}}\binom{K+1}{\frac{K+1}{2}} \\
	  	={}& \frac{K+1}{2^{K+1}}\cdot 
	  	\frac{K+1}{\frac{K+1}{2}}\binom{K}{\frac{K-1}{2}} \\
	  	={}& \frac{K+1}{2^K} \binom{K}{\frac{K+1}{2}} \\
	  	={}& \frac{K+1}{2^K}\cdot 
	  	\frac{K}{\frac{K+1}{2}}\binom{K-1}{\frac{K-1}{2}} \\
	  	={}& \cR(K)\,.
	  \end{align*}
In other words, the regret $\cR(K)$ obeys the following pattern \[
\cR(1) = \cR(2) < \cR(3) = \cR(4) < \dots < \cR(2n-1) = \cR(2n) < \cdots\,.
\]
Therefore, if $ K $ is odd, it holds that \[ 
\cR(K) = \cR(K+1) \le \sqrt{\frac{2(K+1)}{\pi}}\,.
 \]
As a result, for any $ K $, even or odd, the following inequality holds \[ 
\cR(K) \le \sqrt{\frac{2(K+1)}{\pi}}\,.
 \]
By \cref{lem:t-k-inequality}, we obtain \begin{align*}
 \roco(B^1, 
B^{*1}, 
K, T) \le{}& \left\lceil\frac{T}{K}\right\rceil \cR(K)
 \le  \left\lceil\frac{T}{K}\right\rceil \sqrt{\frac{2(K+1)}{\pi}}\\
 \le{}&
\frac{2T}{\sqrt{K(K+1)}} \sqrt{\frac{2(K+1)}{\pi}} = 2\sqrt{\frac{2}{\pi}} 
\frac{T}{\sqrt{K}}\,.
 \end{align*}
Additionally, \cref{prop:non-decrease-in-n} and 
\cref{prop:exact_upper_bound_3d} imply \[ 
	\roco(B^1, 
	B^{*1}, 
	K, T) \le \roco(B^2_2, 
	B^{*2}_2, 
	K, T) \le \left\lceil\frac{T}{K}\right\rceil \sqrt{K}\,.
 \]
As a consequence, we prove part (\ref{it:first-1d-upper-bound}) \[ 
	\roco(B^1, 
	B^{*1}, 
	K, T) \le \left\lceil\frac{T}{K}\right\rceil \min\{ 
	\sqrt{\frac{2(K+1)}{\pi}},  \sqrt{K} \}  
	\le 
	2\sqrt{\frac{2}{\pi}}\frac{T}{\sqrt{K}} < \frac{1.6T}{\sqrt{K}}\,.
 \]

Next, we show part (\ref{it:second-1d-upper-bound}). Notice that if $ K $ is 
odd, \[ 
	\frac{\cR(K+2)/\sqrt{K+2}}{\cR(K)/\sqrt{K}}=\frac{\sqrt{K(K+2)}}{K+1} < 1\,.
 \]
 Hence, for any odd $K\geq 3$, $\cR(K)/\sqrt{K}\leq 
 \cR(3)/\sqrt{3}=\frac{\sqrt{3}}{2}$.
 
 Recall that  $\cR(K)/\sqrt{K}\leq \sqrt{\frac{2}{\pi}}$ for when $K$ is 
 even. Therefore, for all $ K\ge 2 $, we have \[ 
 \cR(K)\leq 
 \frac{\sqrt{3}}{2}\sqrt{K}\leq 0.87\sqrt{K}\,.
  \]  
  Thus we obtain \[ 
  \roco(B^1, B^{*1}, K, T) \le \frac{\sqrt{3}}{2} \left\lceil \frac{T}{K} 
  \right\rceil \sqrt{K} < 0.87 \left\lceil \frac{T}{K} 
  \right\rceil \sqrt{K}\,.
   \]
\end{proof}

We now easily show an upper bound on the minimax rate for the
 $ \infty $-norm by the same idea.

\begin{prop}[Upper bound for $ \infty $-norm]\label{prop:infty-norm-upper-bound}
	The minimax regret $\roco(B^{n}_\infty,B^{*n}_\infty,K,T)$ is at most  
	$2\sqrt{\frac{2}{\pi}}\frac{nT}{\sqrt{K}}$.
\end{prop}
\begin{proof}
	The argument in \cref{prop:infty-norm-high-d} shows \[ 
	\roco(B^{n}_\infty,B^{*n}_\infty,K,T)= 
	n \roco(B^1,B^{*1},K,T)\,.
	 \]
	 The desired upper bound follows from $ \roco(B^1, 
	 B^{*1}, 
	 K, T) \le 2\sqrt{\frac{2}{\pi}}\frac{T}{\sqrt{K}} $ shown in 
	 \cref{prop:upper-bound}.
\end{proof}

\begin{prop}[Unequal block length for $K=3$, informal] \label{unequal_blocks}
The minimax game between player and adversary when $K=3$ has the player choose to make unequally spaced switches. In particular, the first switch happens at approximately $0.29T$ through the game, strictly before $0.33T$.
\end{prop}
\begin{proof}
The proof of \cref{prop:fugal-operator-quadratic} showed that $z_+$ and $z_-$ are $\pm (\sqrt{2}-1)$, via considering Case 2 of the proof and setting $i=2$ and $z=0$. Without loss of generality, we assume the optimal $z$ is $\sqrt{2}-1$ ($z_+$ and $z_-$ are symmetric). Then, we translate the optimal $z$ into the location of the optimal first switch. Plugging $z'=\sqrt{2}-1, z=0, w=1$ into the reparametrization $z'(t) = \frac{Tz+tw}{T-t}$ of \cref{cor:r-lowerbound}, we get $t = (1-\sqrt{2}/2)T \approx 0.29T$. Therefore, the optimal first switch happens at approximately $0.29T$. 
\end{proof}

\end{document}